\documentclass[sigconf]{acmart}

\AtBeginDocument{%
  }

\usepackage{amssymb} 
\usepackage{bm}
\usepackage{bbm}
\usepackage{nicefrac}
\usepackage{siunitx}
\usepackage{enumitem}

\usepackage{graphicx}
\usepackage{svg}
\usepackage{booktabs}
\usepackage{colortbl}
\usepackage{multirow}
\usepackage{multicol}
\usepackage{tabularx}
\usepackage{makecell}
\usepackage{diagbox}
\usepackage{wrapfig}
\usepackage{threeparttable} 
\usepackage{tablefootnote}
\usepackage{float}
\usepackage{balance}
\usepackage{algorithm}
\usepackage{algpseudocode}

\usepackage{pifont}
\usepackage[normalem]{ulem} 
\usepackage{lipsum} 

\usepackage[utf8]{inputenc}

\theoremstyle{acmplain}
\newtheorem{theorem}{Theorem}[section]

\copyrightyear{2025}
\acmYear{2025}
\setcopyright{acmlicensed}\acmConference[MM '25]{Proceedings of the 33rd ACM International Conference on Multimedia}{October 27--31, 2025}{Dublin, Ireland}
\acmBooktitle{Proceedings of the 33rd ACM International Conference on Multimedia (MM '25), October 27--31, 2025, Dublin, Ireland}
\acmDOI{10.1145/3746027.3755331}
\acmISBN{979-8-4007-2035-2/2025/10}
\begin{document}

\title{SpeCa: Accelerating Diffusion Transformers with Speculative Feature Caching}

\author{Jiacheng Liu}
\authornote{Both authors contributed equally to this research.}
\affiliation{
    \institution{Shanghai Jiao Tong University}
    \city{Shanghai}
    \country{China}
}
\affiliation{
  \institution{Shandong University}
  \city{Weihai}
    \country{China}
}

\author{Chang Zou}
\authornotemark[1]
\affiliation{
    \institution{Shanghai Jiao Tong University}
    \city{Shanghai}
    \country{China}
}
\affiliation{
    \institution{University of Electronic Science and Technology of China}
    \city{Chengdu}
    \country{China}
}

\author{Yuanhuiyi Lyu}
\affiliation{
  \institution{The Hong Kong University of Science and Technology (Guangzhou)}
    \city{Guangzhou}
    \country{China}
}

\author{Fei Ren}
\affiliation{
\institution{Tsinghua University}
\city{Beijing}
\country{China}
}

\author{Shaobo Wang}
\affiliation{
\institution{Shanghai Jiao Tong University}
    \city{Shanghai}
    \country{China}
}

\author{Kaixin Li}
\affiliation{
  \institution{National University of Singapore}
  \country{Singapore}
}

\author{Linfeng Zhang}
\authornote{Corresponding author.}
\affiliation{
    \institution{Shanghai Jiao Tong University}
    \city{Shanghai}
    \country{China}
}
\email{zhanglinfeng@sjtu.edu.cn}

\renewcommand{\shortauthors}{Jiacheng Liu, Chang Zou et al.}

\begin{abstract}

Diffusion models have revolutionized high-fidelity image and video synthesis, yet their computational demands remain prohibitive for real-time applications. These models face two fundamental challenges: strict temporal dependencies preventing parallelization, and computationally intensive forward passes required at each denoising step. Drawing inspiration from speculative decoding in large language models, we present \textit{SpeCa}, a novel ``\textit{\textbf{Forecast-then-verify}}'' acceleration framework that effectively addresses both limitations. \textit{SpeCa}'s core innovation lies in introducing Speculative Sampling to diffusion models, predicting intermediate features for subsequent timesteps based on fully computed reference timesteps. Our approach implements a parameter-free verification mechanism that efficiently evaluates prediction reliability, enabling real-time decisions to accept or reject each prediction while incurring negligible computational overhead. 
Furthermore, \textit{SpeCa} introduces sample-adaptive computation allocation that dynamically modulates resources based on generation complexity—allocating reduced computation for simpler samples while preserving intensive processing for complex instances. Experiments demonstrate 6.34$\times$ acceleration on FLUX with minimal quality degradation (5.5\% drop), 7.3$\times$ speedup on DiT while preserving generation fidelity, and 79.84\% VBench score at 6.1$\times$ acceleration for  HunyuanVideo. The verification mechanism incurs minimal overhead (1.67\%-3.5\% of full inference costs), establishing a new paradigm for efficient diffusion model inference while maintaining generation quality even at aggressive acceleration ratios. Our codes have been released in Github: 
\textbf{\href{https://github.com/Shenyi-Z/Cache4Diffusion/}{\texttt{\textcolor{cyan}{https://github.com/Shenyi-Z/Cache4Diffusion/}}}}

\end{abstract}



\settopmatter{printacmref=true}

\maketitle

\section{Introduction}

Diffusion Models (DMs)~\cite{DM} have established themselves as a dominant force in generative AI, achieving state-of-the-art performance across image synthesis~\cite{StableDiffusion} and video generation~\cite{blattmann2023SVD}. While architectural breakthroughs like Diffusion Transformers (DiT)~\cite{DiT} have substantially enhanced generation fidelity through scalable transformer architectures, they remain constrained by an intrinsic limitation of diffusion-based approaches: the sequential dependency in sampling. Each generation requires executing dozens to hundreds of denoising steps, with each step necessitating a complete forward pass through increasingly complex models. This computational paradigm creates an escalating tension between generation quality and inference efficiency - a dilemma that becomes particularly acute as model sizes expand and generation tasks grow more sophisticated. For instance, modern video generation architectures like HunyuanVideo require 595.46 TFLOPs per forward pass (480p, 2s, 50steps), making real-time generation computationally prohibitive.
\begin{figure}
\centering
\includegraphics[width=1\linewidth]{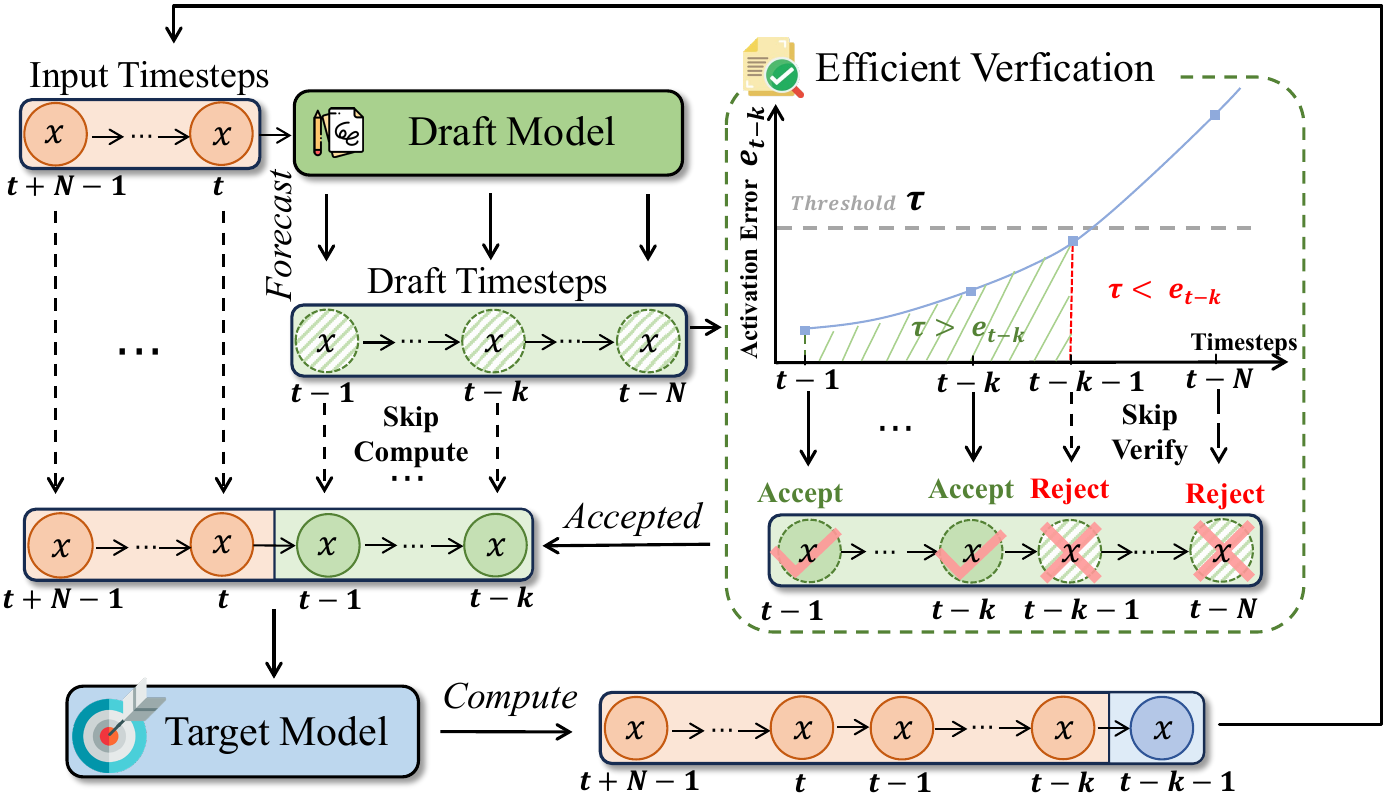}
\caption{\textbf{SpeCa's speculative execution workflow.} The draft model predicts $N$ future timesteps ($t-1$ to $t-N$); lightweight verification checks activation errors. Steps are accepted sequentially until error exceeds $\tau$ at $t-k$, where prediction is rejected. Accepted steps are cached, and the target model resumes computation from $t-k-1$ to ensure fidelity.}
\label{fig:heading}
\vspace{-7mm}
\end{figure}

The computational challenge originates from two fundamental characteristics of diffusion sampling: (1) \textbf{Strict temporal dependencies} enforcing step-by-step execution that prohibits parallelization, and (2) \textbf{Full-model forward passes} required at each timestep, becoming prohibitively expensive for modern architectures. While reduced-step samplers like DDIM~\cite{songDDIM} attempt to address the first challenge through non-Markovian processes, they still face an inescapable trade-off - aggressive step reduction (e.g., from 1000 to 50 steps) inevitably degrades output quality due to trajectory truncation. Recent caching acceleration methods have emerged to tackle primarily the second challenge, with approaches exploring different perspectives - token-based methods ~\cite{selvaraju2024fora}, ~\cite{zou2024accelerating},~\cite{zou2024DuCa} and residual-based approaches ~\cite{chen2024delta-dit} focus on reusing computations across timesteps to reduce per-step computational burden. However, these methods are inherently limited by their reliance on feature similarity between adjacent timesteps, with efficacy rapidly diminishing as acceleration ratios increase. Methods like TaylorSeer~\cite{liuReusingForecastingAccelerating2025} attempt to leverage temporal dependencies more effectively through Taylor series approximations to predict multi-step evolution, but critically lack error correction mechanisms. This absence of validation allows prediction errors to compound exponentially, particularly at high acceleration ratios where minor inaccuracies in early steps catastrophically distort the generation trajectory. Furthermore, existing approaches uniformly apply the same acceleration ratio to all samples and maintain fixed sampling intervals, resulting in computational inefficiency where simple samples receive excessive computation while complex samples remain underserved.

Drawing inspiration from the success of Speculative Decoding in large language models, we present \textit{SpeCa}, an acceleration framework based on a ``\textit{\textbf{Forecast-then-verify}}" mechanism. Unlike traditional step-by-step inference, \textit{SpeCa} adopts a speculative sampling paradigm that implements efficient inference through a three-step core workflow: (1) performing complete forward computations at strategically selected key timesteps to obtain accurate feature representations; (2) predicting features for multiple subsequent timesteps based on these high-fidelity representations; and (3) evaluating the reliability of each predicted feature through a rigorous error verification mechanism at each timestep, dynamically deciding whether to accept or reject predictions. This approach not only addresses the limitations of existing caching methods but also introduces adaptive computation allocation based on sample complexity, providing higher acceleration for simpler samples while ensuring complex samples receive sufficient computational resources to maintain high-fidelity generation quality.

Through extensive experimental evaluation across multiple model architectures and generation tasks, we demonstrate that \textit{SpeCa} significantly outperforms existing acceleration methods in both efficiency and output fidelity. On FLUX.1-dev model, our approach maintains remarkable generation quality with only a 5.5\% quality degradation at 6.34$\times$ acceleration, while the current SOTA method TaylorSeer suffers a substantial 17.5\% quality loss at the same acceleration ratio. Similarly, for DiT model, where competing methods exhibit catastrophic quality deterioration beyond 3$\times$ acceleration, \textit{SpeCa} preserves high-fidelity generation even at an impressive 7.3$\times$ acceleration ratio. Furthermore, when applied to the computationally intensive HunyuanVideo architecture, our method achieves 6.16$\times$ acceleration while maintaining SOTA Vbench scores 79.84\%. Notably, these significant performance gains are achieved with an extremely \textbf{lightweight verification mechanism that consumes merely 3.5\% (DiT), 1.75\% (FLUX), and 1.67\% (HunyuanVideo) of the complete forward pass computation}, yet effectively prevents error accumulation even at aggressive acceleration ratios.

In summary, our contributions are as follows:
\begin{itemize}[leftmargin=10pt,topsep=0pt]
\item \textbf{\textit{SpeCa} Framework}: We propose the \textit{``\textbf{Forecast-then-verify}''} acceleration framework for diffusion models, \textit{SpeCa}, inspired by Speculative Decoding in large language models, addressing the critical efficiency bottleneck in diffusion model inference. This framework, through precise forward prediction and a rigorous lightweight verification mechanism, transcends the theoretical limitations of traditional acceleration methods,  resolving the quality collapse issue at high acceleration ratios.
\item \textbf{Sample-adaptive Computation Allocation}: \textit{SpeCa} dynamically allocates computational steps by sample complexity. \textbf{On HunyuanVideo}, it achieves computation reduction for low-error samples (57.5\% cases accelerated 6.48$\times$) while reserving resources for complex cases (42.5\% cases accelerated 5.82$\times$). This model-specific distribution enables \textbf{1\% VBench quality drop} compared to full computation, outperforming fixed-step methods in quality-efficiency tradeoffs. 
\item \textbf{State-of-the-Art Performance}: Extensive experiments across various model architectures (DiT, FLUX, HunyuanVideo) demonstrate that \textit{SpeCa} significantly outperforms existing acceleration methods. Notably, it achieves only a 5.5\% quality degradation at 6.34$\times$ acceleration on FLUX.1-dev; maintains high-quality generation even at 7.3$\times$ acceleration on DiT; and sustains a SOTA Vbench score of 79.84\% on the computationally intensive HunyuanVideo at 6.1$\times$ acceleration. These results establish \textit{SpeCa} as a new benchmark for efficient diffusion model inference.
\end{itemize}

\begin{figure}
\centering
\includegraphics[width=0.95\linewidth]{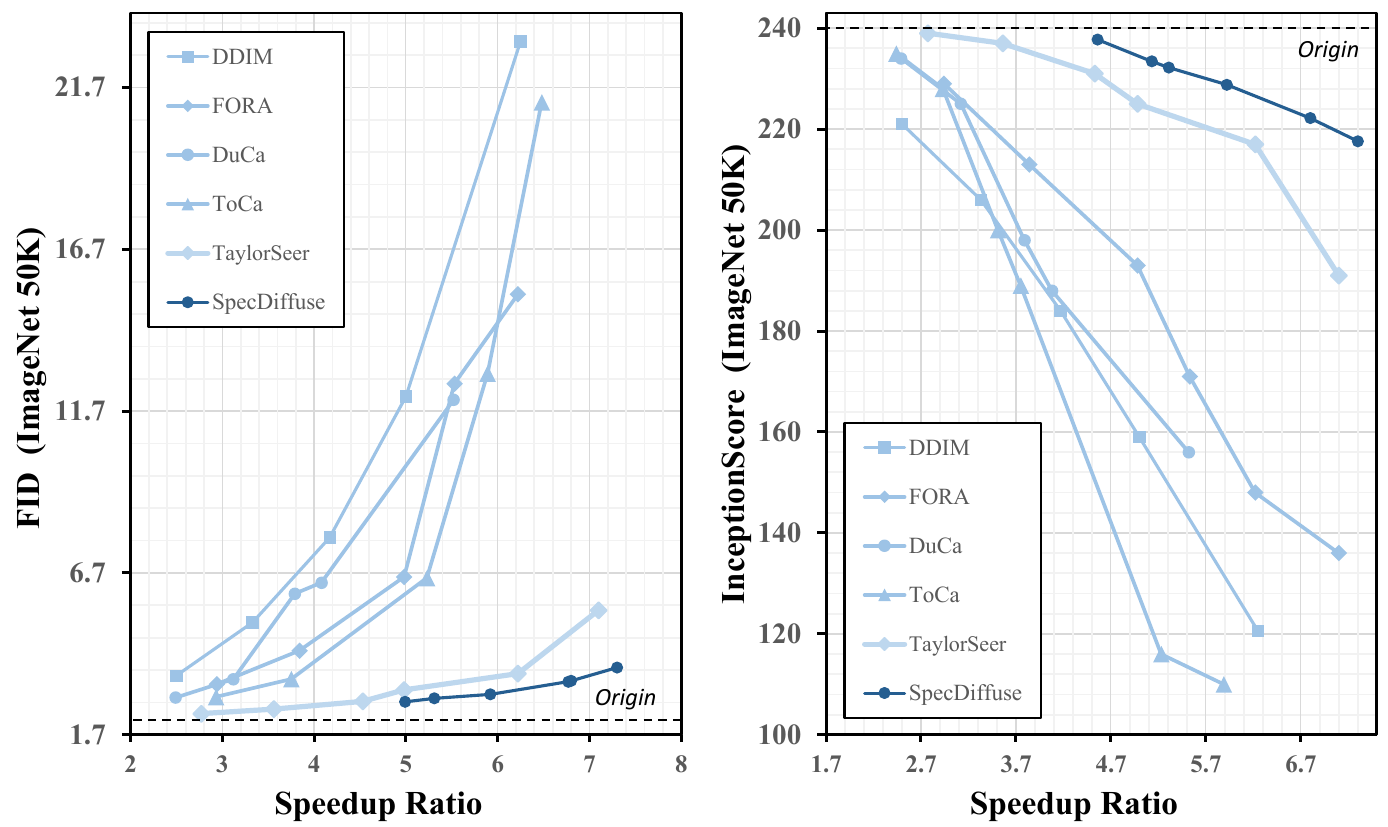}
\caption{\textbf{Comparison of caching methods} in terms of Inception Score (IS) and FID. \textit{SpeCa} achieves superior performance, especially at high acceleration ratios.}
\label{fig:fid-flops}
\vspace{-6mm}
\end{figure}
\vspace{-5mm}

\section{Related Works}\label{sec Works}
Diffusion models~\cite{sohl2015deep,ho2020DDPM} have shown exceptional capabilities in image and video generation. Early architectures, primarily based on U-Net~\cite{ronneberger2015unet}, faced scalability limitations that hindered large model training and deployment. The introduction of Diffusion Transformer (DiT)~\cite{peebles2023dit} overcame these issues, leading to significant advancements and state-of-the-art performance across various domains~\cite{chen2023pixartalpha,chen2024pixartsigma,opensora,yang2025cogvideox}. However, the sequential sampling process in diffusion models remains computationally demanding, prompting the development of acceleration techniques.

\vspace{-4mm}
\subsection{Sampling Timestep Reduction}
\vspace{-1mm}
A key approach to accelerating diffusion models is \textit{minimizing sampling steps while preserving output quality}. DDIM~\cite{songDDIM} introduced a deterministic sampling method that reduced denoising iterations without sacrificing fidelity. The DPM-Solver series~\cite{lu2022dpm,lu2022dpm++,zheng2023dpmsolvervF} further advanced this with high-order ODE solvers. Other strategies, such as Rectified Flow~\cite{refitiedflow} and knowledge distillation~\cite{salimans2022progressive,meng2022on}, reduce the number of denoising steps. Consistency Models~\cite{song2023consistency} enable few-step sampling by directly mapping noisy inputs to clean data, removing the need for sequential denoising.

\vspace{-4mm}
\subsection{Denoising Network Acceleration}
\vspace{-1mm}
To accelerate inference, \textit{optimizing the denoising network's computational efficiency} is crucial. This can be classified into \textit{Model Compression-based} and \textit{Feature Caching-based} techniques.

\vspace{-2mm}
\paragraph{Model Compression-based Acceleration.}
\vspace{-1mm}
Model compression techniques, including network pruning~\cite{structural_pruning_diffusion, zhu2024dipgo}, quantization~\cite{10377259, shang2023post, kim2025ditto}, knowledge distillation~\cite{li2024snapfusion}, and token reduction~\cite{bolya2023tomesd, kim2024tofu}, aim to reduce model complexity while maintaining performance. These methods require retraining or fine-tuning to minimize quality loss and achieve faster inference, though they often involve trade-offs, reducing model size at the cost of expressive power and accuracy.

\vspace{-2mm}
\paragraph{Feature Caching-based Acceleration.}
\vspace{-1mm}
Feature caching is particularly useful for DiT models~\cite{li2023FasterDiffusion, ma2024deepcache}. Techniques like FORA~\cite{selvaraju2024fora} and $\Delta$-DiT~\cite{chen2024delta-dit} reuse attention and MLP representations, while TeaCache~\cite{liu2024timestep} dynamically estimates timestep-dependent differences. DiTFastAttn~\cite{yuan2024ditfastattn} reduces redundancies in self-attention, and ToCa~\cite{zou2024accelerating} updates features dynamically. EOC~\cite{qiu2025acceleratingdiffusiontransformererroroptimized} optimizes using prior knowledge. Innovations like UniCP~\cite{sun2025unicpunifiedcachingpruning} and RAS~\cite{liu2025regionadaptivesamplingdiffusiontransformers} further improve efficiency. However, existing methods rely on a ``\textit{cache-then-reuse}'' paradigm, which loses effectiveness as timestep gaps grow. TaylorSeer~\cite{liuReusingForecastingAccelerating2025} introduced a ``\textit{cache-then-forecast}'' paradigm, predicting future features, but lacks mechanisms to verify prediction accuracy, which may lead to error accumulation.

\begin{figure*}[t]
\centering
\includegraphics[width=\linewidth]{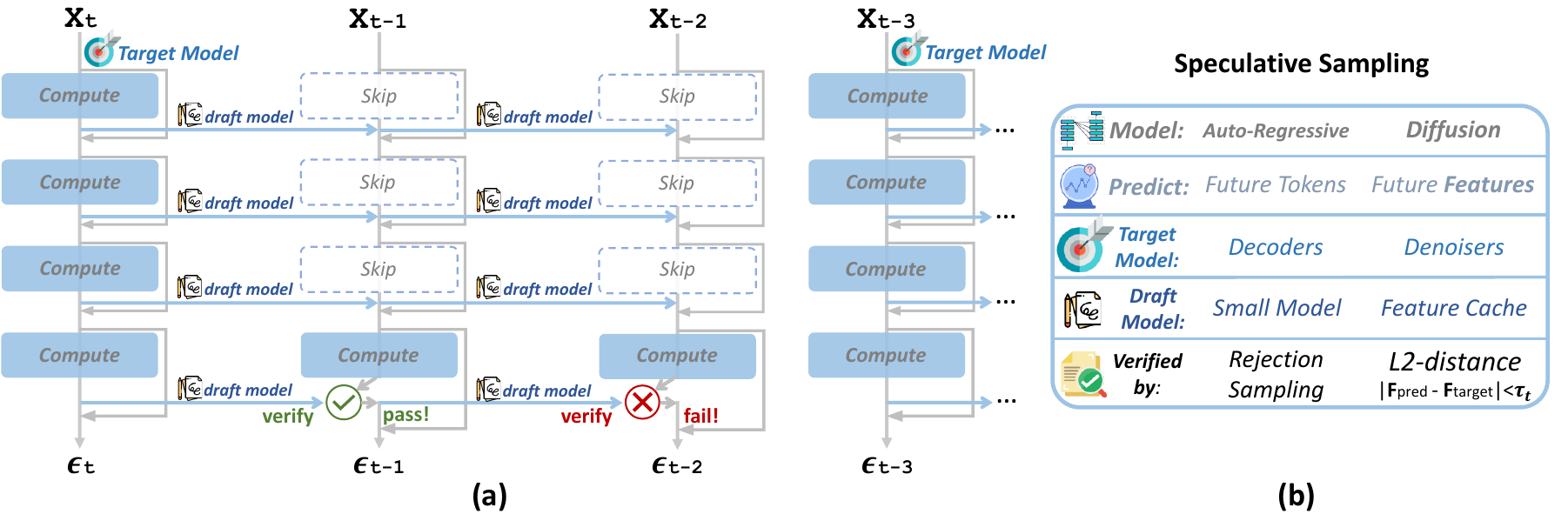}
\vspace{-8mm}
\caption{\textbf{Overview of the SpeCa framework}, which introduces speculative sampling to diffusion models via feature caching. (a) TaylorSeer, a lightweight draft model, predicts future activation features ($t-1$ to $t-N$); a verification module checks L2 error at the final layer. If error exceeds threshold $\tau$ (e.g., at $t-2$), predictions are rejected and full model computation resumes. (b) Comparison with speculative sampling in language models, highlighting key innovations: feature-level prediction, L2-based verification, and adaptive computation. Our method reduces inference cost while maintaining generation quality.}

\vspace{-4mm}
\label{fig:Method}
\end{figure*}

\vspace{-3mm}
\subsection{Speculative Sampling}
\vspace{-1mm}
Speculative decoding has emerged as an effective approach for accelerating large language models (LLMs) while preserving output quality. The core idea, introduced by Leviathan et al.\cite{leviathanFastInferenceTransformers2023}, employs a draft-then-verify mechanism where a smaller model proposes candidate tokens that are efficiently verified in parallel by the main model. Subsequent research has refined this paradigm: SpecInfer\cite{miaoSpecInferAcceleratingGenerative2024} introduced adaptive draft lengths, while Medusa~\cite{caiMedusaSimpleLLM2024} enhanced parallelism through multiple decoding heads. Recent advances have expanded the efficiency frontier of speculative methods. SpecTr~\cite{sunSpecTrFastSpeculative2024} incorporated optimal transport theory for improved batch verification, and Sequoia~\cite{chenSequoiaScalableRobust2024} developed hardware-optimized tree-based verification. These approaches effectively leverage parallel computation to circumvent sequential dependencies.
The speculative paradigm has demonstrated versatility beyond text generation domains. SpecVidGen~\cite{sahariaPhotorealisticTexttoImageDiffusion2022}  successfully adapted these techniques to video and image generation tasks, respectively, establishing the broader applicability of parallel prediction and verification mechanisms across different generative modalities.

Despite these advances, applying speculative sampling to diffusion models presents unique challenges due to the continuous nature of diffusion states and complex feature dependencies across timesteps. Current diffusion acceleration methods primarily rely on caching approaches, which often struggle with error accumulation and quality degradation. This gap motivates our work, which introduces \textit{SpeCa}, a novel framework that bridges feature prediction with speculative principles through a coherent ``\textit{forecast-then-verify}'' mechanism, enabling efficient validation and dynamic acceptance of predicted features across multiple timesteps.

\vspace{-2mm}

\section{Method}
\vspace{-1mm}

\subsection{Preliminary}
\vspace{-1.5mm}

\subsubsection{Diffusion Models.}  
Diffusion models generate structured data by progressively transforming noise into meaningful data through iterative denoising steps. The core mechanism models the conditional probability distribution at each timestep as a Gaussian. Specifically, the model predicts the mean and variance for \(x_{t-1}\) given \(x_t\) at timestep \(t\). The process is expressed as:

\begin{equation}
p_\theta(x_{t-1} | x_t) = \mathcal{N} \left( x_{t-1}; \frac{1}{\sqrt{\alpha_t}} \left( x_t - \frac{1 - \alpha_t}{\sqrt{1 - \bar{\alpha}_t}} \tau_\theta(x_t, t) \right), \beta_t \mathbf{I} \right),
\end{equation}
where \(\mathcal{N}\) denotes a normal distribution, \(\alpha_t\) and \(\beta_t\) are time-dependent parameters, and \(\tau_\theta(x_t, t)\) is the predicted noise at timestep \(t\). The process starts with a noisy image and iteratively refines it by sampling from these distributions until a clean sample is obtained.

\vspace{-1.5mm}
\subsubsection{Diffusion Transformer Architecture.}

The Diffusion Transformer (DiT)~\cite{DiT} employs a hierarchical structure $\mathcal{G} = g_1 \circ g_2 \circ \cdots \circ g_L$, where each module $g_l = \mathcal{F}_{\text{SA}}^l \circ \mathcal{F}_{\text{CA}}^l \circ \mathcal{F}_{\text{MLP}}^l$ consists of self-attention (SA), cross-attention (CA), and multilayer perceptron (MLP) components. In DiT, these components are dynamically adapted over time to accommodate varying noise levels during the image generation process. The input $\mathbf{x}t = \{x_i\}_{i=1}^{H \times W}$ is represented as a sequence of tokens that correspond to image patches. Each module integrates information through residual connections, defined as $\mathcal{F}(\mathbf{x}) = \mathbf{x} + \text{AdaLN} \circ f(\mathbf{x})$, where AdaLN refers to adaptive layer normalization, which facilitates more effective learning.

\vspace{-3mm}
\subsubsection{Speculative Decoding for Diffusion Models}

Diffusion models suffer from high computational cost due to their sequential sampling. To accelerate inference, speculative decoding employs a lightweight draft model to predict a $\gamma$-step future trajectory, which is then verified in parallel by the main model in one forward pass. The resulting effective distribution is:
$
p_{\text{spec}}(\mathbf{x}_0 \mid \mathbf{x}_T) = \sum_{\{\tilde{\mathbf{x}}_{t_i}\}_{i=1}^{\gamma}} \left[ p_{\text{verify}}\!\left( \{\varepsilon_\theta(\tilde{\mathbf{x}}_{t_i}, t_i)\}_{i=1}^{\gamma} \Rightarrow \varepsilon^*_\theta(\mathbf{x}_t, t) \right) \cdot p_{\text{draft}}\!\left( \{\tilde{\mathbf{x}}_{t_i}\}_{i=1}^{\gamma} \mid \mathbf{x}_T \right) \right]
$
Here, $p_{\text{draft}}$ is the draft model's probability of generating the trajectory, and $p_{\text{verify}}$ is the probability the main model accepts it based on noise prediction consistency. This achieves up to a $\gamma$-fold speedup while preserving the main model's output quality.

\vspace{-3mm}
\subsection{SpeCa Framework Overview}
\vspace{-1mm}
Diffusion models incur high computational costs due to their sequential sampling, requiring full computation at each timestep. To mitigate this, we introduce \textit{SpeCa}, a framework that adapts the predict-validate paradigm to accelerate inference.

\textit{SpeCa} intelligently forecasts feature representations for multiple future timesteps and employs a rigorous validation mechanism, significantly reducing computational overhead while preserving generation quality. Unlike existing feature caching methods that often accumulate errors by ignoring temporal dynamics, our approach explicitly models feature evolution across timesteps. The core \textit{SpeCa} workflow follows a ``forecast-verify'' paradigm:

\begin{enumerate}
    \item Execute a complete forward computation at a key time step $t$ to obtain accurate feature representation $\mathcal{F}(x_t^l)$
    \item Based on this feature representation, predict the features for subsequent $k$ time steps $\{t-1, t-2, \ldots, t-k\}$
    \item Evaluate the reliability of the predicted features through an error validation mechanism to decide whether to accept or reject each prediction step
\end{enumerate}
\vspace{-2mm}
This framework generates predicted features using a lightweight draft model and validates them using the main model, forming a dynamically adaptive inference path.

\vspace{-4mm}
\subsection{Draft Model: TaylorSeer Predictor}
\vspace{-1mm}
\textit{SpeCa} uses TaylorSeer as its draft model, which applies Taylor series expansion to predict features across time steps using current features and their derivatives. TaylorSeer requires no training while providing a rigorous and efficient prediction solution.
For a feature $\mathcal{F}(x_t^l)$ at the current time step $t$, TaylorSeer effectively predicts the feature evolution for subsequent $k$ time steps by leveraging temporal patterns in the diffusion process. This mathematical approach captures the underlying dynamics of feature transformation as:
\begin{equation}
\mathcal{F}_{\textrm{pred}}(x_{t-k}^l) = \mathcal{F}(x_t^l) + \sum_{i=1}^{m} \frac{\Delta^i \mathcal{F}(x_t^l)}{i! \cdot N^i}(-k)^i
\end{equation}
\vspace{-1mm}

where $\Delta^i \mathcal{F}(x_t^l)$ represents the $i$-th order finite difference of the feature, capturing essential patterns across sampling points. This term is computed through the expression:
\begin{equation}
\Delta^i \mathcal{F}(x_t^l) = \sum_{j=0}^{i} (-1)^{i-j} \binom{i}{j} \mathcal{F}(x_{t+jN}^l)
\end{equation}
\vspace{-2mm}

This formulation approximates higher-order derivatives with $N$ as the sampling interval and $k$ as the time step difference, enabling accurate and efficient feature prediction.

\vspace{-3mm}
\subsection{Error Computation and Validation}
\vspace{-1mm}
\subsubsection{Relative Error Measure}
We quantify prediction quality using relative error $e_k$, which measures deviation between predicted and actual features:
\begin{equation}
e_k = \frac{|\mathcal{F}_{\text{pred}}(x_{t-k}^l) - \mathcal{F}(x_{t-k}^l)|_2}{|\mathcal{F}(x_{t-k}^l)|_2 + \varepsilon}
\end{equation}
where $\varepsilon=10^{-8}$ prevents division by zero. This metric better reflects generation quality impact than absolute errors.

\subsubsection{Sequential Validation Mechanism}

Our validation follows a sequential decision procedure for predictions $\{t-1, t-2, ..., t-K\}$. At each step, we compare error $e_k$ with threshold $\tau_t$:

\begin{enumerate} 
\item If $e_k \leq \tau_t$, the prediction for that time step is accepted, and validation proceeds to the next step; 
\item If $e_k > \tau_t$, the current and all subsequent predictions are rejected, reverting to standard computation.
\end{enumerate}

Formally, for the prediction sequence $\{t-1, t-2, ..., t-K\}$, the process produces an acceptance set $\mathcal{A}$ and a rejection set $\mathcal{R}$:
\vspace{-1mm}
\begin{align}
\mathcal{A} &= \{t-k \mid e_k \leq \tau_t, 1 \leq k \leq j\} \\
\mathcal{R} &= \{t-k \mid k > j\} \cup \{t-j\} \text{ if } e_j > \tau_t
\end{align}
\vspace{-1mm}
To accommodate varying difficulty across diffusion stages, we apply an adaptive threshold: $\tau_t = \tau_0 \cdot \beta^{\frac{T-t}{T}}$, where $\tau_0$ is initial threshold, $\beta \in (0, 1)$ controls decay, and $T$ is total steps. This allows more speculative execution in early noisy stages while enforcing stricter checks as fine details emerge.

\vspace{-4mm}
\subsection{Computational Complexity Analysis}
\vspace{-1mm}
Diffusion models traditionally require \( T \) sequential sampling steps, each with a complete forward computation, resulting in a total complexity of \( O(T \cdot C) \), where \( C \) represents the per-step computational cost. \textit{SpeCa} significantly reduces this computational burden through a strategic combination of speculative sampling and efficient validation mechanisms. In the \textit{SpeCa} framework, the \( T \) sampling steps are dynamically divided into two categories:

\begin{enumerate}
    \item \textbf{Complete computation steps}: Execute the complete forward computation of the diffusion model
    \item \textbf{Speculative prediction steps}: Use TaylorSeer to predict features and perform lightweight validation
\end{enumerate}

Let $T_{\text{full}}$ and $T_{\text{spec}}$ denote the number of complete computation steps and speculative prediction steps, respectively, with $T = T_{\text{full}} + T_{\text{spec}}$. The computational cost of \textit{SpeCa} can be formally expressed as: $O(T_{\text{full}} \cdot C + T_{\text{spec}} \cdot C_{\text{pred}} + T_{\text{spec}} \cdot C_{\text{verify}})$. $C$ represents the cost of a full computation step, while $C_{\text{pred}}$ and $C_{\text{verify}}$ represent the computational complexity of prediction and validation operations.

Let the prediction acceptance rate be \( \alpha = \frac{T_{\text{spec}}}{T} \), and we express \( T_{\text{full}} = T \cdot (1-\alpha) \). The TaylorSeer predictor incurs negligible overhead (\( C_{\text{pred}} \ll C \)) and requires no additional training. In modern diffusion architectures, validation only involves comparing features in the final layer, costing \( C_{\text{verify}} \approx \gamma \cdot C \), where \( \gamma \ll 1 \) (typically \( \gamma < 0.05 \)). Thus, the total complexity simplifies to:

\vspace{-3mm}
\begin{equation}
O(T \cdot (1-\alpha) \cdot C + T \cdot \alpha \cdot \gamma \cdot C) = O(T \cdot C \cdot (1-\alpha + \alpha \cdot \gamma)) 
\end{equation}

Consequently, the theoretical acceleration ratio of SpeCa over traditional diffusion inference is:
\begin{equation}
\mathcal{S} = \frac{T \cdot C}{T \cdot C \cdot (1-\alpha + \alpha \cdot \gamma)} = \frac{1}{1-\alpha + \alpha \cdot \gamma} 
\end{equation}
Since $\gamma \ll 1$, when the prediction acceptance rate $\alpha$ is relatively high, the speedup can be approximated as: $\mathcal{S} \approx \frac{1}{1-\alpha}$. This efficiency arises from \textit{SpeCa}'s ability to reduce full computations through accurate predictions, while using lightweight validation to maintain generation quality. Unlike methods that require retraining, \textit{SpeCa} can be directly applied to existing diffusion models, providing significant computational savings in practical use.

\vspace{-3mm}

\section{Experiments}

\begin{table*}[htbp]
\centering
\caption{\textbf{Quantitative comparison in text-to-image generation} for FLUX on Image Reward.}
\vspace{-4mm}
\setlength\tabcolsep{5pt}
  \small
  \resizebox{0.98\textwidth}{!}{
  \begin{tabular}{l | c | c  c | c  c | c | c}
    \toprule
    {\bf Method} & {\bf Efficient} &\multicolumn{4}{c|}{\bf Acceleration} &{\bf Image Reward $\uparrow$} &\bf Geneval$\uparrow$ \\
    \cline{3-6}
    {\bf FLUX.1\citep{flux2024}} & {\bf Attention \citep{Dao2022FlashAttentionFA}} & {\bf Latency(s) $\downarrow$} & {\bf Speed $\uparrow$} & {\bf FLOPs(T) $\downarrow$}  & {\bf Speed $\uparrow$} & \bf DrawBench &\bf Overall \\
    \midrule
  
  $\textbf{[dev]: 50 steps}$  & \ding{52}  &  {17.84}  & {1.00$\times$} & {3719.50}   & {1.00$\times$} & {0.9898 \textcolor{gray!70}{\scriptsize (+0.0000)}}  &0.6752 \textcolor{gray!70}{\scriptsize (+0.0000)}      \\ 

  {$60\%$\textbf{ steps}}  & \ding{52}  &  {10.84}  & {1.65$\times$} & {2231.70}   & {1.67$\times$} & {0.9739 \textcolor{gray!70}{\scriptsize (-0.0159)}}  &0.6546 \textcolor{gray!70}{\scriptsize (-0.0206)}     \\

  {$50\%$\textbf{ steps}}   & \ding{52}  &  {9.03}   & {1.98$\times$} & {1859.75}  & {2.00$\times$} & {0.9429 \textcolor{gray!70}{\scriptsize (-0.0469)}}  &0.6655 \textcolor{gray!70}{\scriptsize (-0.0097)}             \\
  {$40\%$\textbf{ steps}}   & \ding{52}  &  {7.15}   & {2.50$\times$} & {1487.80}   & {2.62$\times$} & {0.9317 \textcolor{gray!70}{\scriptsize (-0.0581)}}  &0.6606 \textcolor{gray!70}{\scriptsize (-0.0146)}            \\
  {$34\%$\textbf{ steps}}   & \ding{52}  &  {6.12}   & {2.92$\times$} & {1264.63}   & {3.13$\times$} & {0.9346 \textcolor{gray!70}{\scriptsize (-0.0552)}}  & 0.6501 \textcolor{gray!70}{\scriptsize (-0.0251)}            \\

  {$\Delta$-DiT} ($\mathcal{N}=2$)\citep{chen2024delta-dit}  & \ding{52}  &  {12.07}  & {1.48$\times$} & {2480.01}   & {1.50$\times$} & {0.9316 \textcolor{gray!70}{\scriptsize (-0.0582)}}  &0.6700 \textcolor{gray!70}{\scriptsize (-0.0052)}      \\
  
  {$\Delta$-DiT} ($\mathcal{N}=3$) {\textcolor{red}{†}} \citep{chen2024delta-dit} & \ding{52}  &  {9.21}  & {1.94$\times$} & {1686.76}   & {2.21$\times$} & {0.8561 \textcolor{gray!70}{\scriptsize (-0.1337)}}  & 0.6379 \textcolor{gray!70}{\scriptsize (-0.0373)}      \\
  \midrule
  
  $\textbf{FORA}$ $(\mathcal{N}=6)$ {\textcolor{red}{†}} \citep{selvaraju2024fora}& \ding{52}  &  4.02  & 4.44{$\times$} & {893.54}   & {4.16$\times$} & {0.8235 \textcolor{gray!70}{\scriptsize (-0.1663)}}    &0.5940 \textcolor{gray!70}{\scriptsize (-0.0812)}  \\
  
    $\textbf{\texttt{ToCa}}$ $(\mathcal{N}=8, \mathcal{N}=90\%)$ {\textcolor{red}{†}}  \citep{zou2024accelerating} & \ding{56}  &  5.94  & 3.00{$\times$} & {784.54}   & {4.74$\times$} & {0.9086 \textcolor{gray!70}{\scriptsize (-0.0812)}}   &  0.6347 \textcolor{gray!70}{\scriptsize (-0.0405)}  \\

  $\textbf{\texttt{DuCa}}$ $(\mathcal{N}=8, \mathcal{N}=70\%)$ \citep{zou2024DuCa} & \ding{52}  &  {7.59}    & {2.35$\times$} & {838.61}   & {4.44$\times$} &{0.8905 \textcolor{gray!70}{\scriptsize (-0.0993)}}   & 0.6361 \textcolor{gray!70}{\scriptsize (-0.0391)}  \\                         
  
  \textbf{TeaCache} $({l}=0.8)$ {\textcolor{red}{†}}  \citep{liu2024timestep}& \ding{52}  & 4.98 & 3.58$\times$ & 892.35 & 4.17$\times$  &  0.8683 \textcolor{gray!70}{\scriptsize (-0.1215)} & 0.6356 \textcolor{gray!70}{\scriptsize (-0.0396)}\\ 

  $\textbf{TaylorSeer} $ $(\mathcal{N}=5,O=2)$ \citep{liuReusingForecastingAccelerating2025}& \ding{52}  &    6.17 &  2.89{$\times$} &  {893.54}  &  {4.16$\times$} &0.9992 \textcolor{gray!70}{\scriptsize (+0.0094)}   &0.6266 \textcolor{gray!70}{\scriptsize (-0.0486)} \\

  \rowcolor{gray!20}
  $\textbf{SpeCa} $  & \ding{52}  & 4.93  & 3.62$\times$ &  791.40  & 4.70$\times$ &  \textbf{0.9998 \textcolor[HTML]{0f98b0}{\scriptsize (+0.0100)}} &  \textbf{0.6397 \textcolor[HTML]{0f98b0}{\scriptsize (+0.0355)}}\\

  \midrule

  $\textbf{FORA}$ $(\mathcal{N}=7)$ {\textcolor{red}{†}} \citep{selvaraju2024fora}& \ding{52}  & 3.98   & 4.48{$\times$} &  670.44 & {5.55$\times$} & 0.7398 \textcolor{gray!70}{\scriptsize (-0.2500)}  & 0.5678 \textcolor{gray!70}{\scriptsize (-0.1074)}\\
  
  $\textbf{\texttt{ToCa}}$ $(\mathcal{N}=10,\mathcal{N}=90\%)$ {\textcolor{red}{†}}  \citep{zou2024accelerating} & \ding{56}  &  5.65  & 3.16{$\times$} & 714.66  & {5.20$\times$} &  0.7131 \textcolor{gray!70}{\scriptsize (-0.2767)} & 0.6026 \textcolor{gray!70}{\scriptsize (-0.0726)}\\
  
  $\textbf{\texttt{DuCa},}$ $(\mathcal{N}=9, \mathcal{N}=90\%)$ {\textcolor{red}{†}} \citep{zou2024DuCa} & \ding{52}  &  4.69  & 3.80{$\times$} &  690.26 & {5.39$\times$} & 0.8601 \textcolor{gray!70}{\scriptsize (-0.1297)}  &  0.6189 \textcolor{gray!70}{\scriptsize (-0.0563)}\\      
  
  \textbf{TeaCache}$({l}=1.2)$  {\textcolor{red}{†}}  \citep{liu2024timestep}& \ding{52}  & 3.98 & 4.48{$\times$} &669.27 & 5.56$\times$ &  0.7351 \textcolor{gray!70}{\scriptsize (-0.2547)} & 0.5845 \textcolor{gray!70}{\scriptsize (-0.0907)}\\ 

  $\textbf{TaylorSeer} $ $(\mathcal{N}=7,O=2)$ {\textcolor{red}{†}} \citep{liuReusingForecastingAccelerating2025}& \ding{52}  &   5.02 & 3.55{$\times$}  & 670.44  & {5.55$\times$} & 0.9331 \textcolor{gray!70}{\scriptsize (-0.0567)}  & 0.5886 \textcolor{gray!70}{\scriptsize (-0.0866)}\\
  
  \rowcolor{gray!20}
  
  $\textbf{SpeCa} $& \ding{52}  &  4.14&  4.31{$\times$}  &  \textbf{652.14}  & \textbf{5.70}$\times$ & \textbf{0.9717 \textcolor[HTML]{0f98b0}{\scriptsize (-0.0181)}} & \textbf{0.6338 \textcolor[HTML]{0f98b0}{\scriptsize (-0.0414)}} \\
  \midrule

  $\textbf{FORA}$ $(\mathcal{N}=9)$ {\textcolor{red}{†}} \citep{selvaraju2024fora}& \ding{52}  &   3.36 & 5.31{$\times$} &  596.07 & {6.24$\times$} &  0.5442 \textcolor{gray!70}{\scriptsize (-0.4456)} & 0.5199 \textcolor{gray!70}{\scriptsize (-0.1553)}\\
  
  $\textbf{\texttt{ToCa}}$ $(\mathcal{N}=12, \mathcal{N}=90\%)$ {\textcolor{red}{†}}  \citep{zou2024accelerating} & \ding{56}  &  5.27  & 3.39{$\times$} & 644.70  & {5.77$\times$} &  0.7131 \textcolor{gray!70}{\scriptsize (-0.2767)} &  0.5479 \textcolor{gray!70}{\scriptsize (-0.1273)}\\
  
  $\textbf{\texttt{DuCa}}$ $(\mathcal{N}=12, \mathcal{N}=80\%)$ {\textcolor{red}{†}} \citep{zou2024DuCa} & \ding{52}  & 4.18  & 4.27{$\times$} & 606.91  & {6.13$\times$} &  0.6753 \textcolor{gray!70}{\scriptsize (-0.3145)} & 0.5561 \textcolor{gray!70}{\scriptsize (-0.1191)}\\  
  
  \textbf{TeaCache} $({l}=1.4)$ {\textcolor{red}{†}} \citep{liu2024timestep}& \ding{52}  & 3.63 & 4.91$\times$  & 594.90& 6.25$\times$  & 0.7346 \textcolor{gray!70}{\scriptsize (-0.2552)}  & 0.5524 \textcolor{gray!70}{\scriptsize (-0.1228)}\\ 

  $\textbf{TaylorSeer} $ $(\mathcal{N}=9,O=2)$ {\textcolor{red}{†}} \citep{liuReusingForecastingAccelerating2025}& \ding{52}  & 4.72   & 3.78{$\times$}  & 596.07  & {6.24$\times$} & 0.8168 \textcolor{gray!70}{\scriptsize (-0.1730)}  & 0.5380 \textcolor{gray!70}{\scriptsize (-0.1372)}\\
  \rowcolor{gray!20}
  $\textbf{SpeCa} $  & \ding{52}  &  3.81  &  4.68{$\times$} &  \textbf{586.93} & \textbf{6.34}$\times$& \textbf{0.9355 \textcolor[HTML]{0f98b0}{\scriptsize (-0.0543)}} & \textbf{0.5922 \textcolor[HTML]{0f98b0}{\scriptsize (-0.0830)}}\\

    \bottomrule
  \end{tabular}}
  
  \label{table:FLUX-Metrics}
\footnotesize
\begin{itemize}
\item \textcolor{red}{†} Methods exhibit significant degradation in Image Reward, leading to severe deterioration in image quality.
\item \textcolor{gray!70}{gray}: Deviation from FLUX.1-dev 50stpes baseline (ImageReward=0.9898, Geneval Overall=0.6752). \textcolor[HTML]{0f98b0}{Blue}: \textbf{SpeCa}'s superior performance with minimal quality loss.
\end{itemize}
\vspace{-5mm}
\end{table*}

\begin{table}[htbp]
\centering
\caption{\textbf{Quantitative comparison in text-to-video generation} for HunyuanVideo on VBench.}
\vspace{-4mm}
\setlength\tabcolsep{1pt} 
\small
\resizebox{1.0\linewidth}{!}{
\begin{tabular}{l |c c| c c| c}
\toprule
\textbf{Method} & \textbf{Latency(s) $\downarrow$} & \textbf{Speed $\uparrow$} & \textbf{FLOPs(T) $\downarrow$} & \textbf{Speed $\uparrow$} & \textbf{VBench $\uparrow$} \\
\midrule
\textbf{50-steps} & 145.00 & 1.00$\times$ & 29773.0 & 1.00$\times$ & 80.66 \\ 

\textbf{22\% steps} & 31.87 & 4.55$\times$ & 6550.1 & 4.55$\times$ & 78.74 \\
\midrule
\textbf{TeaCache\textsuperscript{\textcolor{red}{1}}} & 30.49 & 4.76$\times$ & 6550.1 & 4.55$\times$ & 79.36 \\
\textbf{FORA} & 34.39 & 4.22$\times$ & 5960.4 & 5.00$\times$ & 78.83 \\
\textbf{ToCa} & 38.52 & 3.76$\times$ & 7006.2 & 4.25$\times$ & 78.86 \\
\textbf{DuCa} & 31.69 & 4.58$\times$ & 6483.2 & 4.62$\times$ & 78.72 \\
\textbf{TeaCache\textsuperscript{\textcolor{red}{2}}} & 26.61 & 5.45$\times$ & 5359.1 & 5.56$\times$ & 78.32 \\
\textbf{TaylorSeer\textsuperscript{\textcolor{red}{1}}} & 34.84 & 4.16$\times$ & 5960.4 & 5.00$\times$ & 79.93 \\
\rowcolor{gray!20}
\textbf{SpeCa\textsuperscript{\textcolor{red}{1}}} & 34.58 & 4.19$\times$ & \textbf{5692.7} & \textbf{5.23$\times$} & \textbf{79.98} \\
\textbf{TaylorSeer\textsuperscript{\textcolor{red}{2}}} & 31.69 & 4.58$\times$ & 5359.1 & 5.56$\times$ & 79.78 \\
\rowcolor{gray!20}
\textbf{SpeCa\textsuperscript{\textcolor{red}{2}}} & 31.45 & 4.61$\times$ & \textbf{4834.8} & \textbf{6.16$\times$} & \textbf{79.84} \\
\bottomrule
\end{tabular}}
\label{table:HunyuanVideo-Metrics-Compact}
\vspace{-7mm}
\end{table}
\vspace{-1mm}

\subsection{Experiment Settings}
\vspace{-1mm}
\subsubsection{Model Configurations.}
For text-to-image generation, we employ \textbf{FLUX.1-dev}\citep{flux2024}, which utilizes Rectified Flow sampling with 50 steps. This model is evaluated on NVIDIA H800 GPUs. For text-to-video generation, we use \textbf{HunyuanVideo}\citep{li_hunyuan-dit_2024, sun_hunyuan-large_2024}, evaluated on the Hunyuan-Large architecture with 50-step inference, using NVIDIA H800 80GB GPUs. Lastly, for class-conditional image generation, we utilize \textbf{DiT-XL/2}~\cite{DiT}, which employs 50-step DDIM sampling and is evaluated on NVIDIA A800 80GB GPUs. \textit{For detailed model configurations, please refer to the Supplementary Material.}

\begin{figure*}
\centering
\includegraphics[width=\linewidth]{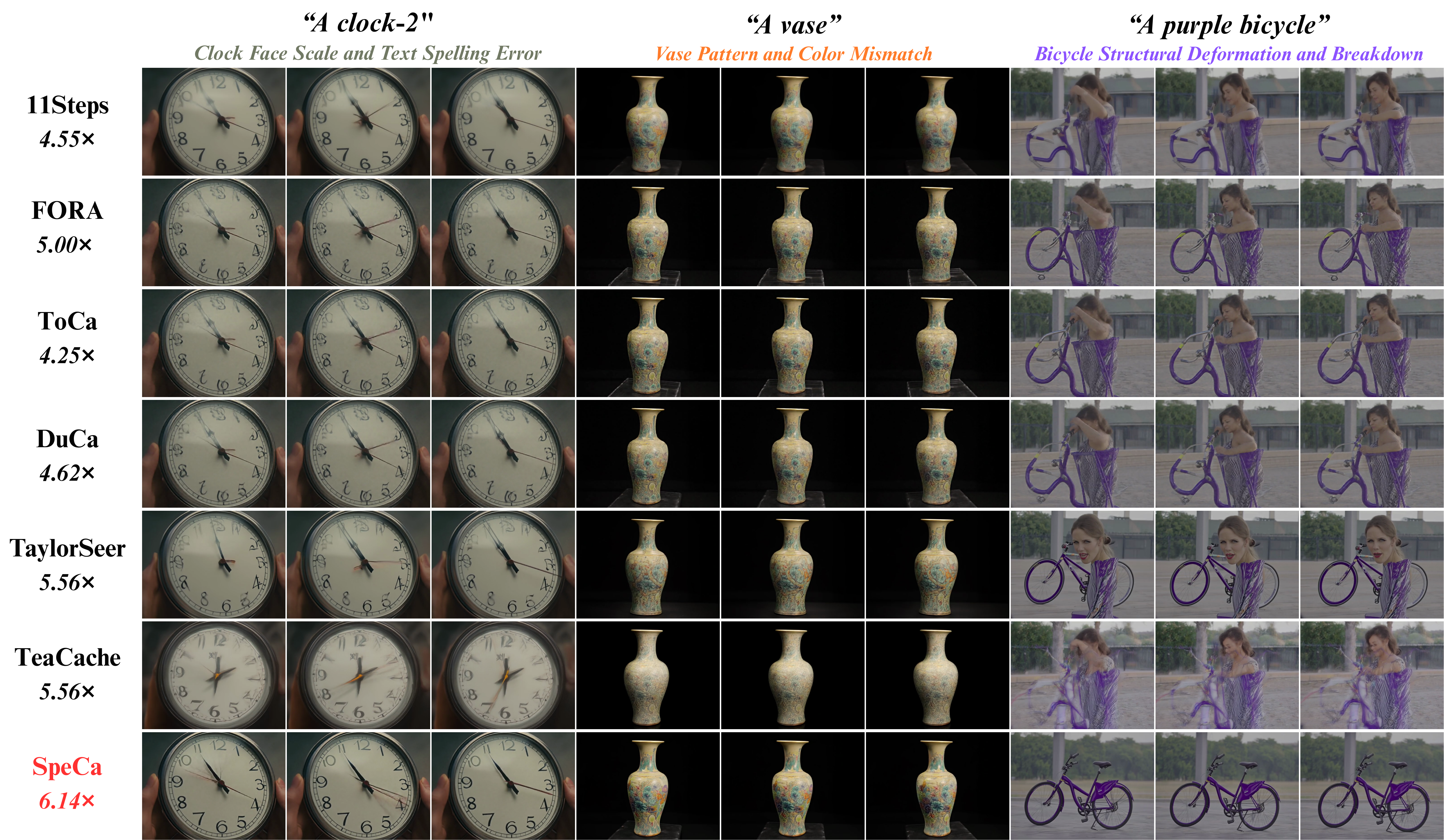}
\vspace{-5mm}
\caption{Comparison of generation methods across objects. While others achieve accuracy but exhibit \textit{spelling errors}, \textit{missing details}, or \textit{distortions}, our method maintains high quality and precision without such flaws.}
\label{fig:Vis-Huanyuan}
\vspace{-4mm}
\end{figure*}
 \vspace{-2mm}

\subsubsection{Evaluation and Metrics.}

For FLUX.1-dev at 1024×1024 resolution, we generated images using 200 high-quality prompts from DrawBench~\cite{sahariaPhotorealisticTexttoImageDiffusion2022}. Image quality was assessed with ImageReward~\cite{xuImageRewardLearningEvaluating2023}, and text alignment using GenEval~\cite{ghoshGenEvalObjectFocusedFramework2023}. For text-to-video, we used VBench~\cite{VBench} with 946 prompts to generate 4,730 videos, evaluated across 16 dimensions. In class-conditional image generation, 50,000 images were sampled across 1,000 ImageNet~\cite{Imagenet} classes and evaluated using FID-50k~\cite{FID50K}, sFID, and InceptionScore.
\begin{figure}
\centering
\includegraphics[width=\linewidth]{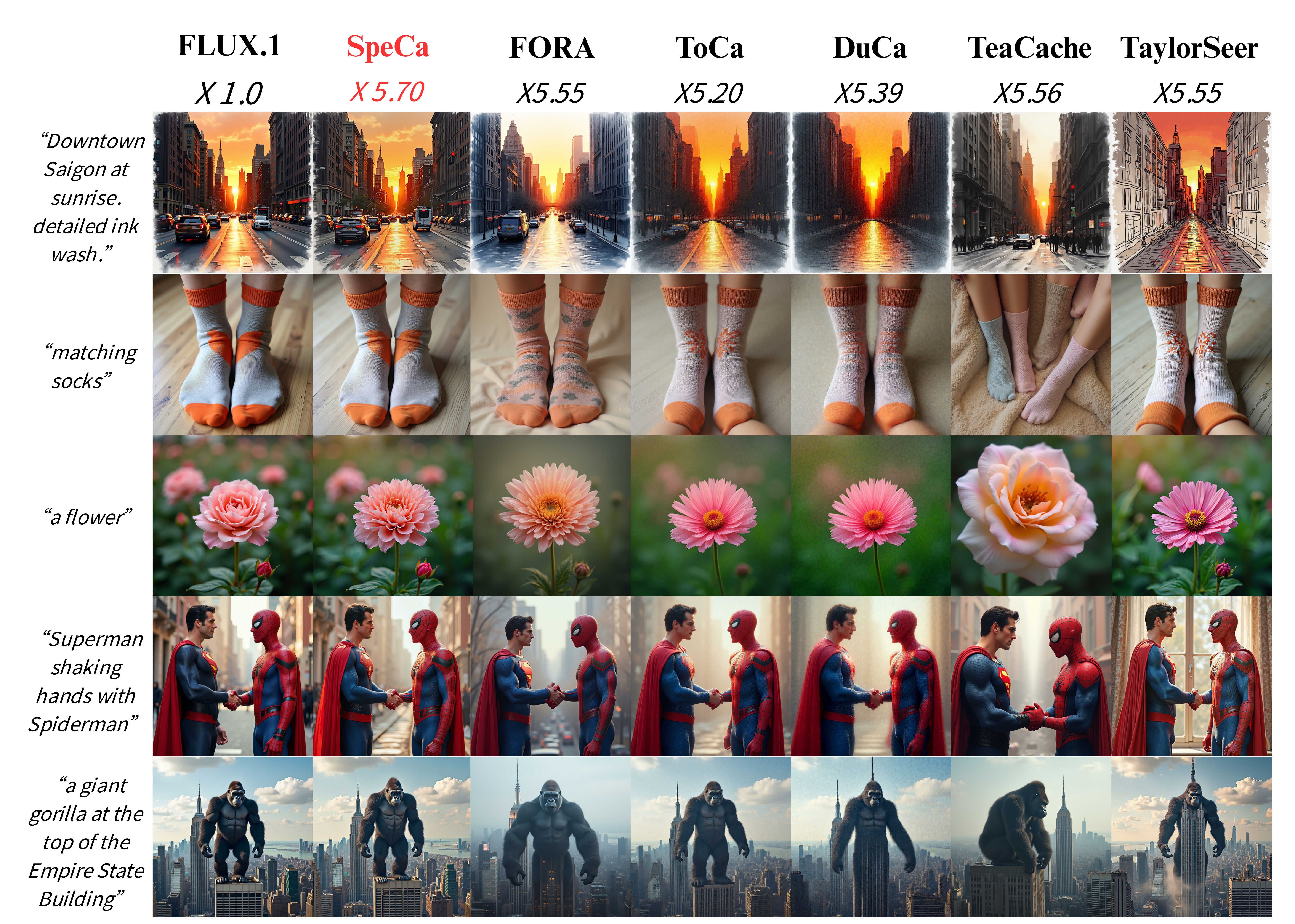}
\vspace{-7mm}
\caption{Text-to-image comparison: SpeCa achieves visual fidelity on par with FLUX.}
\vspace{-11mm}
\label{fig:Vis-FLUX}.
\end{figure}

\vspace{-3mm}
\subsection{Text-to-Image Generation}
\paragraph{Quantitative Study.}
\textit{SpeCa} demonstrates significant superiority over other methods, especially at high acceleration rates. At a 6.34× speedup, \textit{SpeCa} achieves an Image Reward of \textbf{0.9355}, far surpassing other methods such as TeaCache (0.7346) and TaylorSeer (0.8168). Additionally, \textit{SpeCa} achieves a Geneval Overall score of \textbf{0.5922}, also outperforming other high-acceleration methods. These results highlight \textit{SpeCa}'s ability to maintain generation quality while achieving high acceleration, with its advantages becoming more pronounced at higher acceleration rates. In contrast, many traditional methods such as TeaCache and TaylorSeer experience significant degradation in image quality when aiming for high acceleration, leading to a substantial decrease in ImageReward. This performance difference underscores \textit{SpeCa}’s robustness and efficiency in high-speed inference, demonstrating its ability to minimize quality loss while enabling substantial acceleration, making it a clear leader in efficient text-to-image generation.

\vspace{-3mm}
\paragraph{Qualitative Study.}
\textit{SpeCa} maintains visual quality comparable to the baseline model FLUX.1-dev even at an acceleration ratio of \textbf{5.70}$\times$, significantly outperforming other high-acceleration methods. Other models at similar acceleration ratios typically exhibit reduced structural integrity (as shape distortion in ``\textit{matching socks}''), decreased detail precision (like texture blurring in ``\textit{downtown Saigon}''), and object boundary ambiguity (like extra limbs appearing in ``\textit{Superman shaking hands with Spiderman}''). In contrast, \textit{SpeCa} demonstrates remarkable advantages in preserving object morphology, textural details, semantic accuracy, and object boundary integrity, proving its excellence in text-to-image generation.
\vspace{-2mm}
\subsection{Text-to-Video Generation}
\vspace{-1mm}
\paragraph{Quantitative Study.}

\textit{SpeCa} demonstrates exceptional computational efficiency while maintaining high quality. As shown in Table~\ref{table:HunyuanVideo-Metrics-Compact}, our base configuration achieves \textbf{5692.7 TFLOPs} with \textbf{5.23}$\times$ acceleration and \textbf{79.98}\% VBench, outperforming TaylorSeer ($\mathcal{N}$=5, $\mathcal{O}$=1) at 5.00$\times$ acceleration and 79.93\% quality. \textit{SpeCa} outperforms step reduction (78.74\%), TeaCache (79.36\%), ToCa (78.86\%), and DuCa (78.72\%), despite similar acceleration ratios. In the enhanced configuration, computational cost is further reduced to 4834.8 TFLOPs with 6.16$\times$ speedup, while maintaining 79.84\% quality, showcasing the effectiveness of speculative sampling and sample-adaptive computation in balancing efficiency and output quality.

\begin{table*}[htbp]
\centering
\caption{\textbf{Quantitative comparison on class-to-image generation} on ImageNet with \text{DiT-XL/2.}}
\vspace{-4mm}
\setlength\tabcolsep{6pt} 
\small
\begin{tabular}{l | c | c c c | c c |c}
\toprule
\bf Method  & \makecell{\bf Efficient \\ \bf Attention} & \bf Latency(s) $\downarrow$ & \bf FLOPs(T) $\downarrow$ & \bf Speed $\uparrow$  & \bf FID $\downarrow$ & \bf sFID $\downarrow$ &   \makecell{\bf Inception\\ \bf Score} $\uparrow$  \\
\midrule
{\textbf{$\text{DDIM-50 steps}$}} & \ding{52}& {0.995}  & {23.74}  & {1.00$\times$}  &  {{2.32}}\textcolor{gray!70}{\scriptsize (+0.00)} &  {{4.32}}\textcolor{gray!70}{\scriptsize (+0.00)}   &{241.25}\textcolor{gray!70}{\scriptsize (+0.00)}\\
{\textbf{$\text{DDIM-25 steps}$}} & \ding{52}& {0.537}  & {11.87}  & {2.00$\times$}  &  {{3.18 \textcolor{gray!70}{\scriptsize (+0.86)}}} &  {{4.74 \textcolor{gray!70}{\scriptsize (+0.42)}}}   &{232.01 \textcolor{gray!70}{\scriptsize (-9.24)}}\\
{\textbf{$\text{DDIM-20 steps}$}}\textcolor{red}{†} & \ding{52}& {0.406}  & {9.49}  & {2.50$\times$}  &  {{3.81 \textcolor{gray!70}{\scriptsize (+1.49)}}} &  {{5.15 \textcolor{gray!70}{\scriptsize (+0.83)}}} &{221.43 \textcolor{gray!70}{\scriptsize (-19.82)}}\\
{\textbf{$\text{DDIM-12 steps}$}} {\textcolor{red}{†}} & \ding{52}& {0.128}  & {5.70}  & {4.17$\times$}  &  {{7.80 \textcolor{gray!70}{\scriptsize (+5.48)}}} &  {{8.03 \textcolor{gray!70}{\scriptsize (+3.71)}}} &{184.50 \textcolor{gray!70}{\scriptsize (-56.75)}}\\
{\textbf{$\text{DDIM-10 steps}$}} {\textcolor{red}{†}} & \ding{52}& {0.224}  & {4.75}  & {5.00$\times$}  &  {{12.15 \textcolor{gray!70}{\scriptsize (+9.83)}}} &  {{11.33 \textcolor{gray!70}{\scriptsize (+7.01)}}} &{159.13 \textcolor{gray!70}{\scriptsize (-82.12)}}\\
{\textbf{$\text{DDIM-8 steps}$}}{\textcolor{red}{†}}& \ding{52}& {0.189} & 3.80 & 6.25$\times$ & 23.13 \textcolor{gray!70}{\scriptsize (+20.81)} & 19.23 \textcolor{gray!70}{\scriptsize (+14.91)} & 120.58 \textcolor{gray!70}{\scriptsize (-120.67)}\\
{\textbf{$\text{DDIM-7 steps}$}}{\textcolor{red}{†}}& \ding{52}& {0.168} & 3.32 & 7.14$\times$  & 33.65 \textcolor{gray!70}{\scriptsize (+31.33)} & 27.51 \textcolor{gray!70}{\scriptsize (+23.19)}& 92.74 \textcolor{gray!70}{\scriptsize (-148.51)}\\
{\textbf{$\text{$\Delta$-DiT}$}}($\mathcal{N}=2$)  & \ding{52}& {0.740}  & {18.04}  & {1.31$\times$}  &  {{2.69 \textcolor{gray!70}{\scriptsize (+0.37)}}} &  {{4.67 \textcolor{gray!70}{\scriptsize (+0.35)}}} &{225.99 \textcolor{gray!70}{\scriptsize (-15.26)}}\\
{\textbf{$\text{$\Delta$-DiT}$}}($\mathcal{N}=3$) \textcolor{red}{†}  & \ding{52}& {0.658}  & {16.14}  & {1.47$\times$}  &  {3.75 \textcolor{gray!70}{\scriptsize (+1.43)}} &  {5.70 \textcolor{gray!70}{\scriptsize (+1.38)}} &{207.57 \textcolor{gray!70}{\scriptsize (-33.68)}}\\

\midrule     

\textbf{FORA} ($\mathcal{N}=6$) {\textcolor{red}{†}}& \ding{52} & 0.427 & 5.24 & {4.98$\times$}  & 9.24 \textcolor{gray!70}{\scriptsize (+6.92)} & 14.84 \textcolor{gray!70}{\scriptsize (+10.52)}  &{171.33 \textcolor{gray!70}{\scriptsize (-69.92)}}\\

\textbf{\texttt{ToCa}} ($\mathcal{N}=9,\mathcal{N}=95\%$)  \textcolor{red}{†} & \ding{56}& {1.016} & 6.34 & 3.75$\times$  & 6.55 \textcolor{gray!70}{\scriptsize (+4.23)} & 7.10 \textcolor{gray!70}{\scriptsize (+2.78)} & 189.53 \textcolor{gray!70}{\scriptsize (-51.72)} \\

\textbf{\texttt{DuCa}} ($\mathcal{N}=6,\mathcal{N}=95\%$) \textcolor{red}{†} & \ding{52}& {0.817} & 5.81 & 4.08$\times$  & 6.40 \textcolor{gray!70}{\scriptsize (+4.08)} & 6.71 \textcolor{gray!70}{\scriptsize (+2.39)} & 188.42 \textcolor{gray!70}{\scriptsize (-52.83)} \\

$\textbf{TaylorSeer} $ $(\mathcal{N}=6,O=4)$ & \ding{52}& {0.603} &4.76& 4.98$\times$  & 3.09 \textcolor{gray!70}{\scriptsize (+0.77)}& 6.50 \textcolor{gray!70}{\scriptsize (+2.18)}& 225.16 \textcolor{gray!70}{\scriptsize (-16.09)} \\

\rowcolor{gray!20}
\textbf{SpeCa} & \ding{52}& {0.476} & \textbf{4.76} & \textbf{4.99}$\times$  & \textbf{2.72 \textcolor[HTML]{0f98b0}{\scriptsize (+0.40)}} & \textbf{5.51 \textcolor[HTML]{0f98b0}{\scriptsize (+1.19)}} & \textbf{233.85 \textcolor[HTML]{0f98b0}{\scriptsize (-7.40)}} \\

\midrule 

\textbf{FORA} ($\mathcal{N}=7$) {\textcolor{red}{†}}& \ding{52}& {0.405} & 3.82 & 6.22$\times$ & 12.55 \textcolor{gray!70}{\scriptsize (+10.23)} & 18.63 \textcolor{gray!70}{\scriptsize (+14.31)} &148.44 \textcolor{gray!70}{\scriptsize (-92.81)} \\

\textbf{\texttt{ToCa}} ($\mathcal{N}=13,\mathcal{N}=95\%$)  \textcolor{red}{†} & \ding{56}& {1.051} & 4.03 & 5.90$\times$ & 21.24 \textcolor{gray!70}{\scriptsize (+18.92)} & 19.93 \textcolor{gray!70}{\scriptsize (+15.61)} & 116.08 \textcolor{gray!70}{\scriptsize (-125.17)} \\

\textbf{\texttt{DuCa}} ($\mathcal{N}=12,\mathcal{N}=95\%$){\textcolor{red}{†}}  & \ding{52}& {0.698} & 3.94 & 6.02$\times$ & 31.97 \textcolor{gray!70}{\scriptsize (+29.65)} & 27.26 \textcolor{gray!70}{\scriptsize (+22.94)}& 87.94 \textcolor{gray!70}{\scriptsize (-153.31)}\\
$\textbf{TaylorSeer} $ $(\mathcal{N}=8,O=4)$ {\textcolor{red}{†}} & \ding{52}& {0.580} & 3.82 & 6.22$\times$  & 4.42 \textcolor{gray!70}{\scriptsize (+2.10)}& 7.75 \textcolor{gray!70}{\scriptsize (+3.43)}& 205.16 \textcolor{gray!70}{\scriptsize (-36.09)} \\

\rowcolor{gray!20}
\textbf{SpeCa} & \ding{52}& {0.436} & \textbf{3.51} & \textbf{6.76}$\times$  & \textbf{3.76 \textcolor[HTML]{0f98b0}{\scriptsize (+1.44)}} & \textbf{7.13 \textcolor[HTML]{0f98b0}{\scriptsize (+2.81)}} & \textbf{217.10 \textcolor[HTML]{0f98b0}{\scriptsize (-24.15)}} \\

\midrule 

\textbf{FORA} ($\mathcal{N}=8,\mathcal{N}=95\%$) {\textcolor{red}{†}}& \ding{52}& {0.405} & 3.34 & 7.10$\times$  & 15.31 \textcolor{gray!70}{\scriptsize (+12.99)}& 21.91 \textcolor{gray!70}{\scriptsize (+17.59)}& 136.21 \textcolor{gray!70}{\scriptsize (-105.04)} \\

\textbf{\texttt{ToCa}} ($\mathcal{N}=13,\mathcal{N}=98\%$)  \textcolor{red}{†} & \ding{56}& {1.033} & 3.66 & 6.48$\times$  & 22.18 \textcolor{gray!70}{\scriptsize (+19.86)}& 20.68 \textcolor{gray!70}{\scriptsize (+16.36)}& 110.91 \textcolor{gray!70}{\scriptsize (-130.34)} \\

\textbf{\texttt{DuCa}} ($\mathcal{N}=18,\mathcal{N}=95\%$){\textcolor{red}{†}}  & \ding{52}& {0.714} & 3.59 & 6.61$\times$  & 133.06 \textcolor{gray!70}{\scriptsize (+130.74)}& 98.13 \textcolor{gray!70}{\scriptsize (+93.81)} & 15.87 \textcolor{gray!70}{\scriptsize (-225.38)} \\
$\textbf{TaylorSeer} $ $(\mathcal{N}=9,O=4)${\textcolor{red}{†}} & \ding{52}& {0.571} & 3.34 & 7.10$\times$  & 5.55 \textcolor{gray!70}{\scriptsize (+3.23)}& 8.45 \textcolor{gray!70}{\scriptsize (+4.13)}& 191.19 \textcolor{gray!70}{\scriptsize (-50.06)} \\

\rowcolor{gray!20}
\textbf{SpeCa} & \ding{52}& {0.431} & \textbf{3.26} & \textbf{7.30}$\times$  & \textbf{3.78 \textcolor[HTML]{0f98b0}{\scriptsize (+1.46)}} & \textbf{6.36 \textcolor[HTML]{0f98b0}{\scriptsize (+2.04)}} & \textbf{217.61 \textcolor[HTML]{0f98b0}{\scriptsize (-23.64)}} \\

\bottomrule
\end{tabular}
\label{table:DiT_Metrics}
\vspace{0mm}

\footnotesize
\begin{itemize}
\item \textcolor{red}{†} Methods exhibit significant degradation in FID, leading to severe deterioration in image quality.
\item \textcolor{gray!70}{gray}: Deviation from DDIM-50 baseline (FID=2.32, sFID=4.32, IS=241.25). \textcolor[HTML]{0f98b0}{Blue}: \textbf{SpeCa}'s superior performance with minimal quality loss.
\end{itemize}
\vspace{-3mm}
\end{table*}
\begin{figure}
\centering
\includegraphics[width=\linewidth]{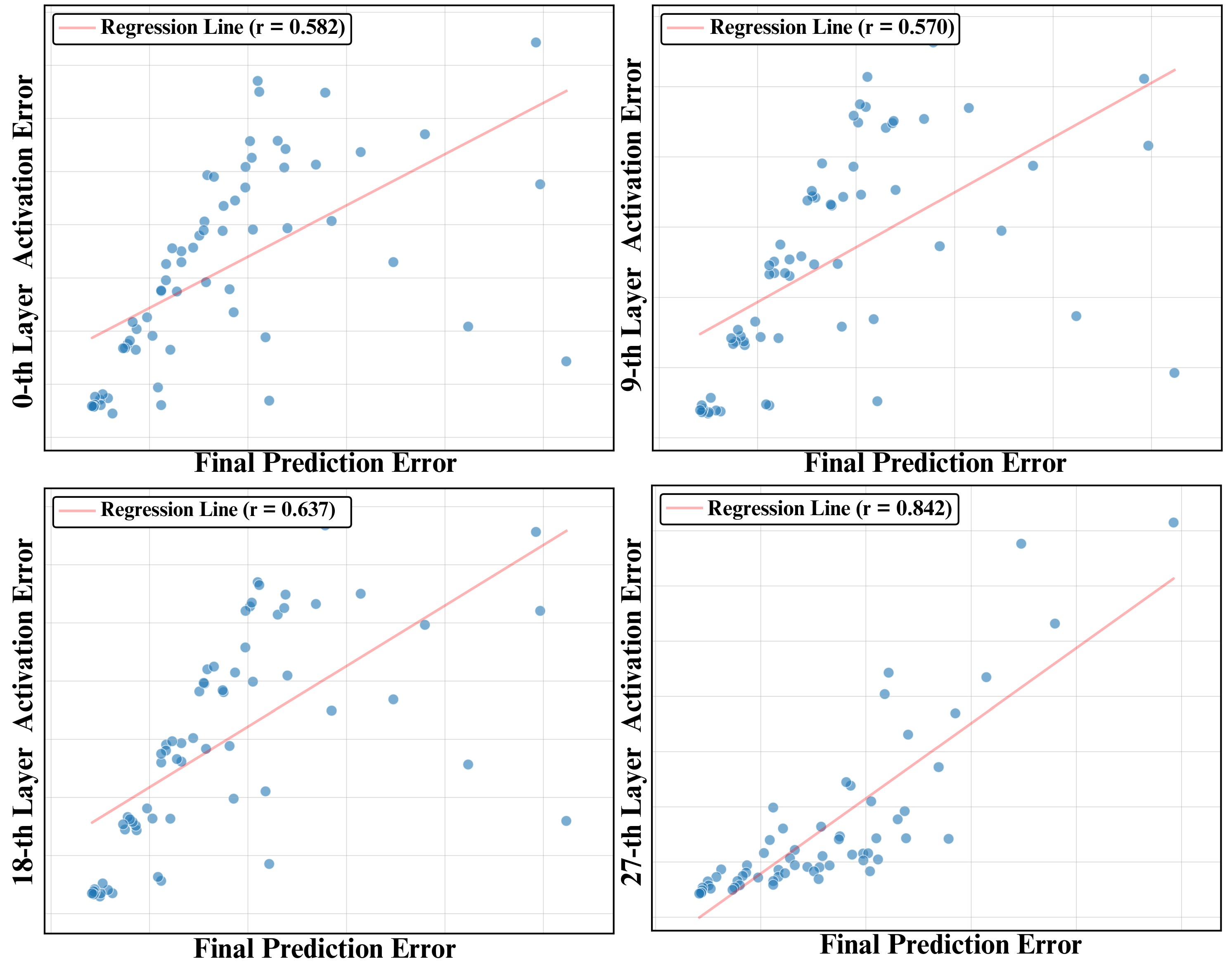}
\vspace{-8mm}
\caption{Strong correlation  between errors at layer 27 and final output, validating it as an effective monitoring point.}
\label{fig:Error}
\vspace{-9mm}
\end{figure}

\vspace{-2mm}
\paragraph{Qualitative Study.}
\textit{SpeCa} effectively mitigates three common generation errors while accelerating inference. First, it eliminates clock face scaling and text spelling errors, accurately rendering symbols and text with high visual and semantic fidelity. Second, for vases with complex patterns, \textit{SpeCa} preserves fine details and color accuracy—where others exhibit blurring or mismatching. Third, under high acceleration, \textit{SpeCa} maintains structural integrity in challenging cases like bicycles, avoiding deformation and motion artifacts that plague baseline methods. Thanks to its speculative sampling framework, \textit{SpeCa} achieves fast, stable, and accurate generation.

\vspace{-4mm}
\subsection{Class-Conditional Image Generation}
\vspace{-1mm}
We evaluated \textit{SpeCa} against established acceleration methods including ToCa, FORA, DuCa, state-of-the-art TaylorSeer, and baseline DDIM with reduced sampling steps on DiT. Results demonstrate that \textit{SpeCa} consistently outperforms existing methods, particularly at high acceleration ratios. At \textbf{5$\times$} acceleration, conventional methods struggle: DDIM-10 (FID 12.15), FORA (9.24), ToCa (12.86), and DuCa (12.05) show substantial degradation, while SpeCa achieves an exceptional FID of only \textbf{2.72}, outperforming TaylorSeer (3.09) by 12\% with 5$\times$ faster inference. The disparity becomes more pronounced at \textbf{6.76 $\times$
} acceleration, where traditional approaches severely degrade: DDIM-8 (FID 23.13), ToCa (21.24), and DuCa (31.97). Even TaylorSeer shows significant quality loss (FID 4.42). Under these extreme conditions, \textit{SpeCa} maintains remarkable robustness with an FID of only \textbf{3.76} while preserving advantages across sFID and Inception Score. These results empirically validate our method's superiority, effectively mitigating error accumulation in diffusion model acceleration.

\vspace{-4mm}
\subsection{Ablation Studies}
\vspace{-1mm}

\begin{figure*}
\centering
\begin{minipage}{0.30\linewidth}
\centering
\includegraphics[width=0.98\linewidth]{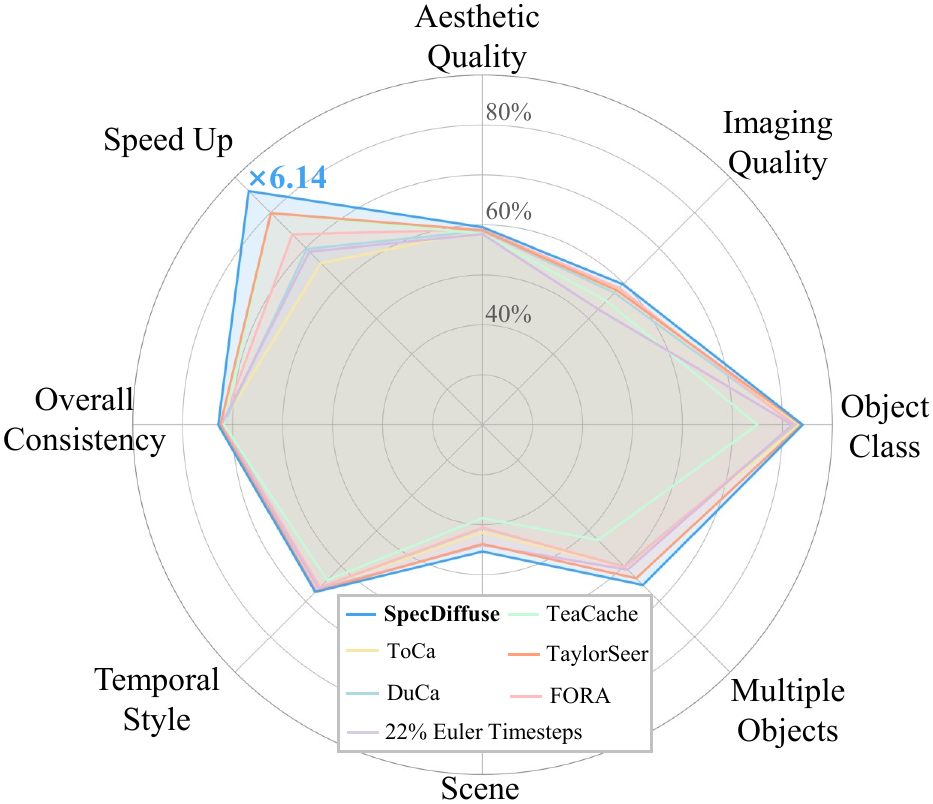}
\vspace{-3mm}
\caption{VBench performance of \textit{SpeCa} versus baselines.}
\label{fig:vbench}
\end{minipage}
\hfill
\begin{minipage}{0.68\linewidth}
\centering
\includegraphics[width=0.96\linewidth]{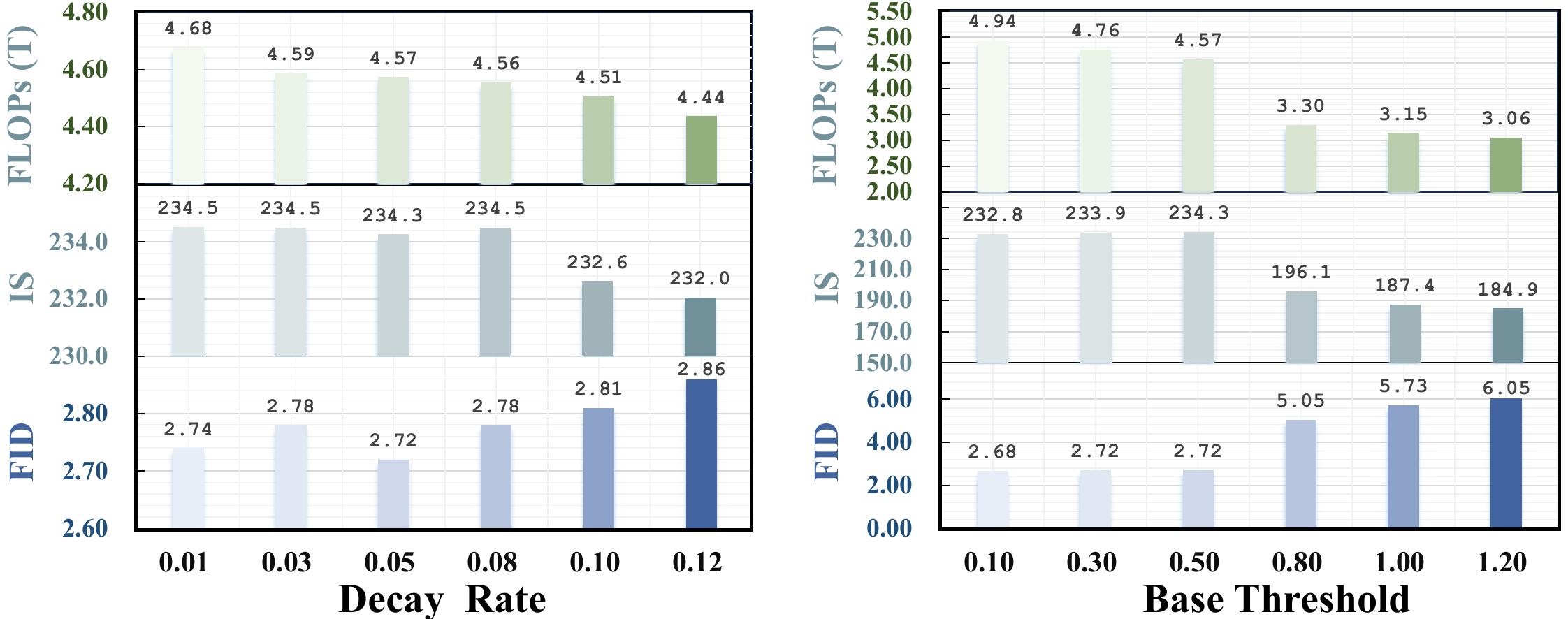}
\vspace{-4mm}
\caption{Hyperparameter sensitivity analysis of SpeCa showing effects of base ratio ($\tau_0$) and decay rate ($\beta$) on computational efficiency and generation quality in SpeCa.}
\label{fig:ablation}
\end{minipage}
\vspace{-6mm}
\end{figure*}

\paragraph{Hyperparameter Analysis}
Figure~\ref{fig:ablation} demonstrates the impact of threshold ($\tau_0$) and decay rate ($\beta$) on \textit{SpeCa} performance. Our analysis reveals that increasing $\tau_0$ substantially reduces computational demands (measured in FLOPs) at the expense of generation quality, with FID and sFID metrics showing marked degradation beyond $\tau_0 > 0.5$. Meanwhile, higher decay rates ($\beta$) yield marginal improvements in acceleration with negligible impact on output fidelity. Our adaptive threshold scheduling enables more aggressive time step skipping in the early denoising stages and more conservative progression in later stages, effectively balancing efficiency and fidelity throughout the generation process. Tables~\ref{table:ablation_decay} and \ref{table:ablation_threshold} provide comprehensive quantification of these relationships, illustrating how these parameters enable optimal trade-offs between computational efficiency and generation quality.
\vspace{-3mm}
\paragraph{Validation Layer Selection and Trajectory Analysis}
Our analysis of layer-wise feature correlations reveals that deeper layers provide more reliable indicators of final output quality. Figure~\ref{fig:Error} presents correlation analysis between activation errors measured at different DiT network layers and final generated image errors. Each scatter plot represents a generated sample, with the X-axis depicting the error between the final predicted output and the original full-step diffusion model output, while the Y-axis represents the activation error at specific network layers. Results demonstrate that 27-th Layer (deep layer) activation errors exhibit the \textbf{strongest correlation ($r$=0.842)} with final output errors, substantially higher than shallow layers and middle layers. This finding provides robust support for our validation strategy: we can efficiently predict final generation quality by monitoring deep layer feature errors without computing the entire network. This aligns precisely with our theoretical analysis of error propagation based on Taylor expansion, confirming that deeper features have a more direct and deterministic influence on final output quality. Additionally, trajectory analysis in feature space confirms that \textit{SpeCa} maintains evolution paths closely aligned with standard sampling, even at high acceleration ratios (5$\times$). This trajectory preservation explains why our method maintains generation quality while significantly reducing computation. The consistent correlation patterns observed across diverse visual domains further validate our approach's generalizability and robustness in practical generation scenarios. \textit{Refer to Appendices for detailed analyses.}
\vspace{-4mm}

\section{Conclusion}
\vspace{-2mm}

In this paper, we introduce \textit{SpeCa}, a novel acceleration framework based on the ``\textbf{forecast-then-verify}'' mechanism for diffusion models. By extending speculative sampling principles to the diffusion domain, we successfully address the quality collapse issue that plagues traditional methods at high acceleration ratios. \textit{SpeCa} employs a lightweight verification mechanism to effectively evaluate prediction reliability, while introducing sample-adaptive computation allocation that dynamically adjusts computational resources based on generation complexity. Extensive experiments demonstrate significant performance improvements across various architectures: achieving 6.34$\times$ acceleration with only 5.5\% quality degradation on FLUX models, maintaining high-fidelity generation at 7.3$\times$ acceleration on DiT models, and attaining 79.84\% VBench score at 6.1$\times$ acceleration on the computationally intensive HunyuanVideo. As a plug-and-play solution, \textit{SpeCa} seamlessly integrates into existing diffusion architectures without additional training, opening new avenues for diffusion model deployment in resource-constrained environments. Future work will explore applications to additional modalities and investigate synergies with other acceleration techniques, further enhancing the efficiency and applicability of diffusion models in practical scenarios.

\vspace{-2mm}
\begin{acks}
\vspace{-1mm}
This work was partially supported by the Dream Set Off - Kunpeng\&Ascend Seed Program.
\end{acks}
\bibliographystyle{ACM-Reference-Format}
\balance
\bibliography{sample-base}

\clearpage

\appendix
\newpage
\appendix
\section{Experimental Details}
In this section, more details of the experiments are provided. All models incorporate a unified forced activation period $\mathcal{N}$, while $\mathcal{O}$ represents the order of the Taylor expansion.
\subsection{Model Configurations.}
The experiments are conducted on three state-of-the-art visual generative models: the text-to-image generation model FLUX.1-dev~\citep{flux2024}, text-to-video generation model HunyuanVideo~\citep{sun_hunyuan-large_2024}, and the class-conditional image generation model DiT-XL/2~\citep{DiT}. 

\noindent \textbf{FLUX.1-dev}\citep{flux2024} predominantly employs the Rectified Flow~\citep{refitiedflow} sampling method with 50 steps as the standard configuration. All experimental evaluations of FLUX.1-dev were conducted on NVIDIA H800 GPUs.

\noindent \textbf{HunyuanVideo}~\citep{li_hunyuan-dit_2024, sun_hunyuan-large_2024} was evaluated on the Hunyuan-Large architecture, utilizing the standard 50-step inference protocol as the baseline while preserving all default sampling parameters to ensure experimental rigor and consistency. Comprehensive performance benchmarks were systematically conducted on NVIDIA H800 80GB GPUs for detailed latency assessment and thorough inference performance analysis. The configuration parameters for FORA were set to $\mathcal{N}=5$; ToCa and DuCa were configured with $\mathcal{N}=5, R=90\%$; TaylorSeer\textsuperscript{\textcolor{red}{1}} was parameterized with $\mathcal{N}=5, \mathcal{O}=1$, while TaylorSeer\textsuperscript{\textcolor{red}{2}} used $\mathcal{N}=6, \mathcal{O}=1$; TeaCache\textsuperscript{\textcolor{red}{1}} was configured with a threshold of $l=0.4$, and TeaCache\textsuperscript{\textcolor{red}{2}} with $l=0.5$.

\noindent \textbf{DiT-XL/2}~\cite{DiT} adopts a 50-step DDIM~\citep{songDDIM} sampling strategy to ensure consistency with other models. Experiments on DiT-XL/2 are conducted on NVIDIA A800 80GB GPUs.

\section{Hyperparameter Analysis}
Figure~\ref{fig:ablation} illustrates how the base threshold ($\tau_0$) and decay rate ($\beta$) affect \textit{SpeCa} performance. Increasing the base threshold significantly reduces computational cost (FLOPs) while elevating FID and sFID values, with quality metrics notably deteriorating when $\tau_0>0.5$. Higher decay rates slightly improve acceleration ratios with minimal impacts on generation quality. Tables~\ref{table:ablation_decay} and~\ref{table:ablation_threshold} further quantify these trends, demonstrating how adjusting these parameters enables \textit{SpeCa} to achieve various trade-offs between computational efficiency and generation quality.

\section{Trajectory Analysis}

Figure~\ref{fig:Trajectory} presents feature evolution trajectories of different acceleration methods during the generation process. These trajectories are visualized in a two-dimensional plane after applying Principal Component Analysis (PCA) to high-dimensional feature representations at each timestep. Each trajectory represents a specific sample (such as ``\textit{daisy}'', ``\textit{horse}'', ``\textit{tiger}'', or "\textit{flamingo}") progressing from noise to clear image, with star and square markers indicating the trajectory's beginning and end points, respectively. The visualization clearly demonstrates that despite achieving a 4.99$\times$ acceleration, \textit{SpeCa}'s feature evolution trajectory (red line) almost perfectly overlaps with the original DiT (blue line, 1.0$\times$ speed), while significantly outperforming other acceleration methods such as ToCa (4.92$\times$) and TaylorSeer (4.99$\times$). This result conclusively demonstrates our method's capability to accurately maintain the dynamic characteristics of the original diffusion process while substantially reducing computational steps, ensuring high-quality image generation.

Collectively, these hyperparameter analyses, deep layer validation strategies, and trajectory fidelity studies showcase \textit{SpeCa}'s exceptional performance and flexibility in accelerating diffusion model inference, providing robust guarantees for efficient generation in practical applications.

\begin{figure}
\centering
\includegraphics[width=1\linewidth]{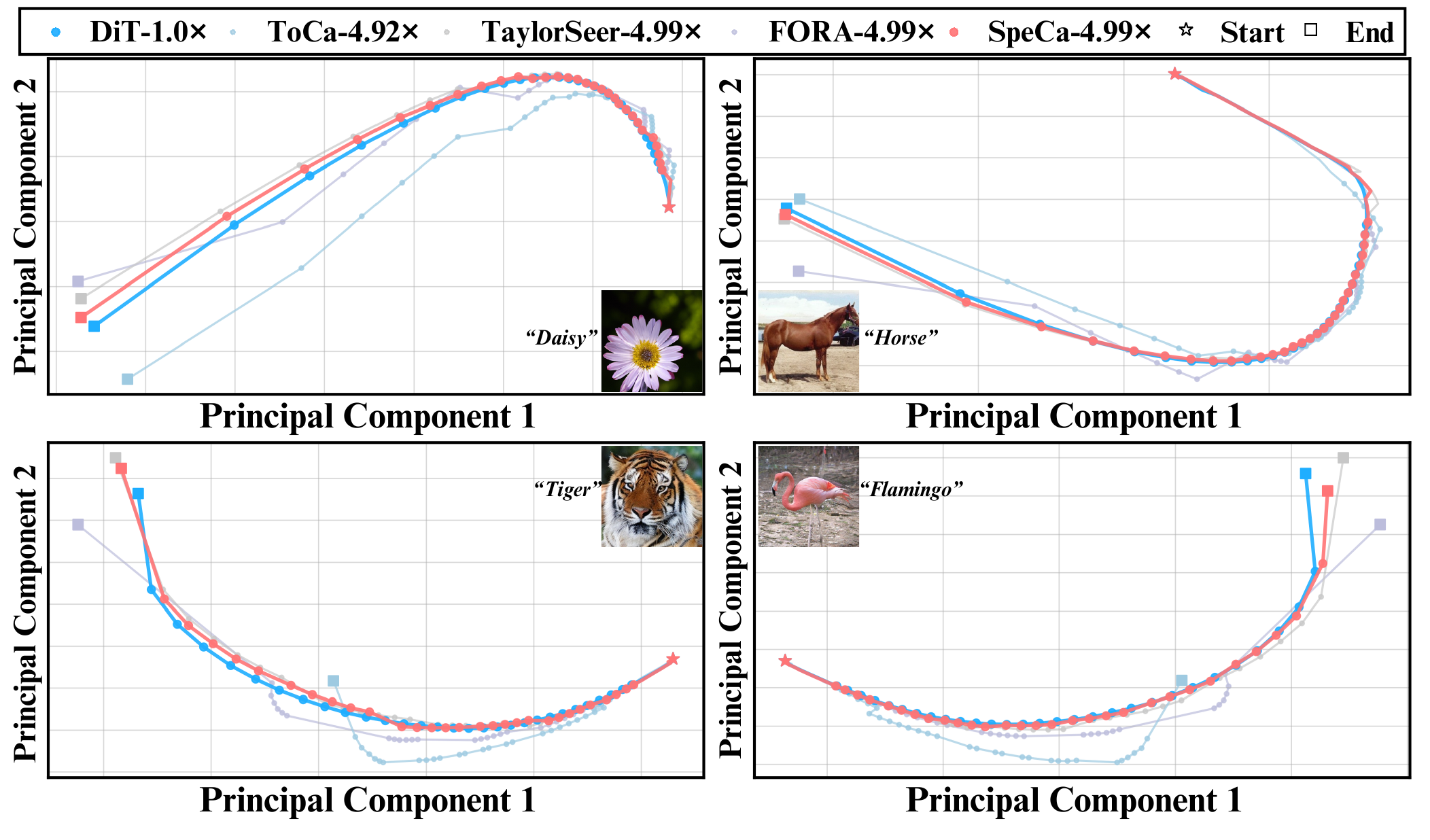}
\caption{Scatter plot of the trajectories of different diffusion acceleration methods after performing Principal Component Analysis (PCA). The figure illustrates how the trajectories evolve across different methods, highlighting their relative efficiencies in terms of feature evolution.}
\label{fig:Trajectory}
\vspace{-3mm}
\end{figure}

\begin{table}[htbp]
\centering
\caption{Ablation study on decay rate (base\_threshold=0.5)}
\vspace{-3mm}
\setlength\tabcolsep{4pt}

\begin{tabular}{ccc|ccc}
\toprule
\textbf{Decay Rate $\beta$} & \textbf{FLOPs(T) $\downarrow$} & \textbf{Speed $\uparrow$} & \textbf{FID $\downarrow$} & \textbf{sFID $\downarrow$} & \textbf{IS $\uparrow$} \\
\midrule
\textbf{0.12} & 4.44 & 5.35 & 2.87 & 5.34 & 232.05 \\
\textbf{0.10} & 4.51 & 5.27 & 2.81 & 5.30 & 232.63 \\
\textbf{0.08} & 4.56 & 5.21 & 2.78 & 5.20 & 234.51 \\
\textbf{0.05} & 4.57 & 5.19 & 2.72 & 5.21 & 234.26 \\
\textbf{0.03} & 4.59 & 5.17 & 2.78 & 5.20 & 234.51 \\
\textbf{0.01} & 4.68 & 5.08 & 2.74 & 5.48 & 234.54 \\
\bottomrule
\end{tabular}
\label{table:ablation_decay}
\end{table}

\begin{table}[htbp]
\centering
\caption{Ablation study on base threshold (decay\_rate=0.5)}
\vspace{-3mm}
\setlength\tabcolsep{4pt}

\begin{tabular}{ccc|ccc}
\toprule
\textbf{Threshold $\tau_0$} & \textbf{FLOPs(T) $\downarrow$} & \textbf{Speed $\uparrow$} & \textbf{FID $\downarrow$} & \textbf{sFID $\downarrow$} & \textbf{IS $\uparrow$} \\
\midrule
\textbf{0.1} & 4.95 & 4.80 & 2.68 & 5.34 & 232.79 \\
\textbf{0.3} & 4.76 & 4.99 & 2.72 & 5.51 & 233.85 \\
\textbf{0.5} & 4.57 & 5.19 & 2.72 & 5.21 & 234.26 \\
\textbf{0.8} & 3.31 & 7.18 & 5.05 & 7.12 & 196.07 \\
\textbf{1.0} & 3.15 & 7.53 & 5.73 & 7.77 & 187.40 \\
\textbf{1.2} & 3.06 & 7.75 & 6.05 & 8.20 & 184.94 \\
\bottomrule
\end{tabular}
\label{table:ablation_threshold}
\end{table}

\subsection{Ablation Study on Validation Layers}
\label{sec:appendix_ablation_layers}

We investigated the impact of the validation layer's depth on generation quality under a fixed $5\times$ speed-up setting. Our results, summarized in Table~\ref{table:ablation_layers}, indicate that validating at deeper layers consistently improves performance. This finding supports our analysis that deeper layers, which accumulate more prediction errors throughout the denoising process, serve as more reliable indicators of the final output's fidelity.

Further evidence is presented in Section 4.5, where we show that the final layer (Layer 27) exhibits the strongest correlation ($r = 0.842$) with the image-level reconstruction error. This significantly outperforms both shallow layers (Layer 0: $r = 0.582$) and middle layers (Layer 18: $r = 0.637$). Consequently, these results justify our choice to use features from deep layers for efficient and accurate verification.
\begin{table}[htbp]
\centering
\caption{\textbf{Ablation of the Verify layer in DiT at $5\times$ speed.}}
\vspace{-0.2cm}
\setlength\tabcolsep{10pt}
\small
\begin{tabular}{l | c c c}
\toprule
\textbf{Method} & \textbf{FID} $\downarrow$ & \textbf{sFID} $\downarrow$ & \makecell{\textbf{Inception}\\\textbf{Score} $\uparrow$} \\
\midrule
\textbf{Layer0 (First)} & 2.81 & 5.67  & 232.88  \\
\textbf{Layer8} & 2.73 & 5.50  & 232.53  \\
\textbf{Layer18} & 2.70 & 5.32  & 233.78  \\
\rowcolor{gray!20}
\textbf{Layer27 (Last)} & 2.61 & 5.24  & 234.98  \\
\bottomrule
\end{tabular}
\label{table:ablation_layers}
\end{table}

\section{Ablation Study on Draft Model Design}
\label{sec:appendix_ablation_draft_model}

We ablated different draft model designs under a fixed $5.1\times$ acceleration on the FLUX model (see Table~\ref{table:draft_ablation}). When compared to direct feature reuse (equivalent to SpeCa without a predictive draft model), both the Adams–Bashforth predictor and our proposed Taylor expansion (\textit{TaylorSeer}) yield significantly better performance.

Notably, \textit{TaylorSeer} achieves the highest CLIP Score (19.32) and ImageReward (1.010), demonstrating its superior accuracy and stability. The Adams–Bashforth method also performs well within our framework (CLIP: 18.80; ImageReward: 0.961) and shows a marked improvement over a non-adaptive baseline (CLIP: 18.30; ImageReward: 0.888). These results confirm that the SpeCa framework enhances prediction accuracy through its combination of structured verification and adaptive draft models.

\begin{table}[htbp]
\centering
\caption{\textbf{Ablation study of  draft models on FLUX.} }
\vspace{-0.2cm}
\setlength\tabcolsep{6pt}
\small
\begin{tabular}{l | c c c}
\toprule
\textbf{Method} & \textbf{CLIP Score} $\uparrow$ & \textbf{ImageReward} $\uparrow$ \\

\midrule

\textbf{Adams-Bashforth (w/o SpeCa)} & 18.30 & 0.888 \\
\midrule
\textbf{SpeCa(w/o TaylorSeer)}  & 18.68 & 0.923 \\
\textbf{Speca(Adams-Bashforth) }      & 18.80 & 0.961 \\
\rowcolor{gray!20}
\textbf{Speca(TaylorSeer) }                & 19.32 & 1.010 \\

\bottomrule
\end{tabular}
\label{table:draft_ablation}
\end{table}

\section{Choice of Error Metrics for Verification}

We evaluated several error metrics for the verification process, including the $\ell_1$, $\ell_2$, and $\ell_\infty$ norms, as well as Cosine Similarity. The experiment was conducted on FLUX with a fixed $5.23\times$ acceleration. As shown in Table~\ref{table:error_metrics_ablation}, the $\ell_2$-norm achieved the best overall performance, with a CLIP Score of 19.28 and an ImageReward of 0.984. Due to its stability and consistently superior results, we adopted the $\ell_2$-norm as our default metric.

Furthermore, we employ a relative error formulation. This approach normalizes discrepancies by the magnitude of the feature vectors, ensuring scale invariance across different denoising steps where activation levels can vary significantly. Using relative error improves the robustness and fairness of the verification, preventing any bias toward early or late stages of the generation process.

Finally, while perceptual metrics like LPIPS could offer greater semantic sensitivity, they introduce substantial computational overhead. This conflicts with the core design principle of speculative execution, which requires the verification step to be lightweight and efficient.

\begin{table}[htbp]
\centering
\caption{\textbf{Ablation of Error metrics on FLUX}}  
\vspace{-0.2cm}
\setlength\tabcolsep{10pt}
\begin{tabular}{l | c c}
\toprule
\textbf{Error Metric} & \textbf{CLIP Score $\uparrow$} & \textbf{ImageReward $\uparrow$} \\
\midrule
Cosine Similarity &19.23&0.979\\
LPIPS&19.23&0.978\\
$\ell_\infty$-norm & 19.18 & 0.946 \\
$\ell_1$-norm & 18.88 & 0.912 \\
\rowcolor{gray!20}
$\ell_2$-norm & \textbf{19.28} & \textbf{0.984} \\

\bottomrule
\end{tabular}
\label{table:error_metrics_ablation}
\end{table}

\subsection{Relationship to Noise Schedules}
A key advantage of SpeCa is its independence from the noise schedule's functional form. Unlike traditional acceleration methods often tailored to specific schedules (e.g., linear or cosine), SpeCa operates by evaluating predictive consistency directly in the latent feature space. This decouples the acceleration from the analytical properties of the noise progression.

Furthermore, SpeCa employs a dynamic step-sizing mechanism, contrasting with the fixed or predetermined steps in methods like DDIM and DPM-Solver. It adaptively adjusts step sizes based on the local smoothness of the feature manifold and the reliability of predictions. This inherent flexibility allows SpeCa to be applied across models with diverse noise schedules without requiring re-tuning or prior knowledge of the noise mechanism, making it a more general and broadly applicable acceleration framework.

\section{Anonymous Page for Video Presentation}
To further showcase the advantages of \textit{SpeCa} in video generation, we have created an anonymous GitHub page. For a more detailed demonstration, please visit \url{https://speca2025.github.io/SpeCa2025/}. Additionally, the videos are also available in the Supplementary Material.

\section{Theoretical Analysis}
\subsection{Error Propagation Analysis}
\begin{theorem}[Prediction Error Bound]
For an $m$-th order Taylor expansion predictor with sampling interval $N$, the prediction error at step $k$ satisfies:
\begin{equation}
\|e_k\|_2 \leq \frac{L^{m+1}k^{m+1}}{(m+1)!N^{m+1}}\|\mathcal{F}(x_t^l)\|_2 + \mathcal{O}\left(\frac{k^{m+2}}{(m+2)!N^{m+2}}\right)
\end{equation}
where $L$ is the Lipschitz constant of the feature evolution.
\end{theorem}
\begin{proof}
The Taylor expansion of $\mathcal{F}(x_{t-k}^l)$ around $x_t^l$ gives:
\begin{align}
\mathcal{F}(x_{t-k}^l) &= \sum_{i=0}^m \frac{(-k)^i}{i!N^i}\Delta^i\mathcal{F}(x_t^l) + R_m(k) \\
\text{where } R_m(k) &= \frac{(-k)^{m+1}}{(m+1)!N^{m+1}}\mathcal{F}^{(m+1)}(\xi),\ \xi\in[t-k,t]
\end{align}

For diffusion models, the feature evolution is governed by the underlying stochastic differential equation:
\begin{equation}
d\mathcal{F}(x_\tau^l) = \mu(\mathcal{F}(x_\tau^l), \tau) d\tau + \sigma(\tau) dW_\tau
\end{equation}
where $\mu$ is the drift term, $\sigma$ is the diffusion coefficient, and $W_\tau$ is the Wiener process.

Due to the smoothness properties of the diffusion process, we can bound the higher-order derivatives. Let $\mathcal{F}^{(i)}$ denote the $i$-th derivative of the feature with respect to time. For the drift and diffusion coefficients in diffusion models, we have:
\begin{equation}
\|\mu(\mathcal{F}(x_\tau^l), \tau)\|_2 \leq L\|\mathcal{F}(x_\tau^l)\|_2
\end{equation}
where $L$ is the Lipschitz constant of the drift term.

Using the theory of finite differences, the approximation error of the finite difference operator is bounded by:
\begin{equation}
\|\Delta^i\mathcal{F}(x_t^l) - N^i\mathcal{F}^{(i)}(x_t^l)\|_2 \leq CN^{i+1}\|\mathcal{F}^{(i+1)}(x_t^l)\|_2
\end{equation}
for some constant $C$.

Furthermore, under the Lipschitz continuity assumption, we can establish:
\begin{equation}
\|\Delta^i\mathcal{F}(x_t^l)\|_2 \leq (2L)^i\|\mathcal{F}(x_t^l)\|_2
\end{equation}

This bound follows from the recursive application of the Lipschitz property across successive timesteps. For the $i$-th derivative, we have:
\begin{equation}
\|\mathcal{F}^{(i)}(x_t^l)\|_2 \leq L^i\|\mathcal{F}(x_t^l)\|_2
\end{equation}

Applying these bounds to the remainder term:
\begin{equation}
\|R_m(k)\|_2 \leq \frac{k^{m+1}}{(m+1)!N^{m+1}}L^{m+1}\|\mathcal{F}(x_t^l)\|_2
\end{equation}

The prediction error is thus:
\begin{align}
\begin{split}
\|\mathcal{F}_{\text{pred}}(x_{t-k}^l) - \mathcal{F}(x_{t-k}^l)\|_2 &= \|R_m(k)\|_2 \\
&\leq \frac{L^{m+1}k^{m+1}}{(m+1)!N^{m+1}}\|\mathcal{F}(x_t^l)\|_2 \\&+ \mathcal{O}\left(\frac{k^{m+2}}{(m+2)!N^{m+2}}\right)
\end{split}
\end{align}

The relative error is defined as:
\begin{align}
e_k &= \frac{\|\mathcal{F}_{\text{pred}}(x_{t-k}^l) - \mathcal{F}(x_{t-k}^l)\|_2}{\|\mathcal{F}(x_{t-k}^l)\|_2 + \tau} \\
\end{align}

To relate $\|\mathcal{F}(x_{t-k}^l)\|_2$ to $\|\mathcal{F}(x_t^l)\|_2$, we use the reverse Lipschitz condition:
\begin{equation}
\|\mathcal{F}(x_{t-k}^l)\|_2 \geq (1 - \frac{Lk}{N})\|\mathcal{F}(x_t^l)\|_2
\end{equation}
which follows from the fact that the feature evolution in diffusion models exhibits controlled decay properties.

Substituting this lower bound:
\begin{align}
e_k &\leq \frac{L^{m+1}k^{m+1}}{(m+1)!N^{m+1}}\frac{\|\mathcal{F}(x_t^l)\|_2}{(1 - \frac{Lk}{N})\|\mathcal{F}(x_t^l)\|_2} + \mathcal{O}(k^{m+2}) \\
&= \frac{L^{m+1}k^{m+1}}{(m+1)!N^{m+1}}\frac{1}{1 - \frac{Lk}{N}} + \mathcal{O}(k^{m+2})
\end{align}

For small values of $\frac{Lk}{N}$ (which is typically the case in our setting), we can approximate $\frac{1}{1 - \frac{Lk}{N}} \approx 1 + \frac{Lk}{N}$, yielding:
\begin{align}
e_k &\leq \frac{L^{m+1}k^{m+1}}{(m+1)!N^{m+1}}\left(1 + \frac{Lk}{N}\right) + \mathcal{O}(k^{m+2}) \\
&= \frac{L^{m+1}k^{m+1}}{(m+1)!N^{m+1}} + \frac{L^{m+2}k^{m+2}}{(m+1)!N^{m+2}} + \mathcal{O}(k^{m+2}) \\
&= \frac{L^{m+1}k^{m+1}}{(m+1)!N^{m+1}} + \mathcal{O}\left(\frac{k^{m+2}}{N^{m+2}}\right)
\end{align}

This completes the proof of the error bound.
\end{proof}

\subsection{Convergence Guarantee}

\begin{theorem}[Distributional Convergence]
When the verification threshold schedule satisfies:
\begin{equation}
\tau_t = \tau_0\cdot\beta^{\frac{T-t}{T}} \leq \sqrt{\frac{\beta_t(1-\bar{\alpha}_t)}{1 + \|\mathcal{F}(x_t^l)\|_2^2}}
\end{equation}
the generated distribution $p_{\text{Spec}}$ converges to $p_{\text{orig}}$ in total variation distance:
\begin{equation}
\lim_{T\to\infty} \|p_{\text{Spec}} - p_{\text{orig}}\|_{\text{TV}} = 0
\end{equation}
\end{theorem}
\begin{proof}
We decompose the proof into three main components: first establishing the relationship between feature errors and output image errors, then analyzing the error propagation through the diffusion process, and finally proving convergence in the distribution space.

\textbf{1. Noise Schedule Compatibility:}

In diffusion models, the denoising process at timestep $t$ is given by:
\begin{equation}
x_{t-1} = \frac{1}{\sqrt{\alpha_t}}\left(x_t - \frac{1-\alpha_t}{\sqrt{1-\bar{\alpha}_t}}\tau_\theta(x_t, t)\right) + \sigma_t z
\end{equation}
where $z \sim \mathcal{N}(0, I)$ is the random noise added during sampling, and $\sigma_t = \sqrt{\beta_t}$.

Let $x_{t-1}^*$ be the output generated using the original model and $x_{t-1}$ be the output from our \textit{SpeCa} framework. The error between these outputs can be expressed as:
\begin{equation}
\begin{split}
\|x_{t-1} - x_{t-1}^*\|_2 &= \left\|\frac{1}{\sqrt{\alpha_t}}\left(x_t - \frac{1-\alpha_t}{\sqrt{1-\bar{\alpha}_t}}\tau_\theta(x_t,t)\right)\right. \\
&\quad \left. - \frac{1}{\sqrt{\alpha_t}}\left(x_t - \frac{1-\alpha_t}{\sqrt{1-\bar{\alpha}_t}}\tau_\theta^*(x_t,t)\right)\right\|_2 \\
&= \frac{1}{\sqrt{\alpha_t}}\frac{1-\alpha_t}{\sqrt{1-\bar{\alpha}_t}}\|\tau_\theta^*(x_t,t) - \tau_\theta(x_t,t)\|_2
\end{split}
\end{equation}

The predicted noise $\tau_\theta(x_t, t)$ is a function of the feature representation $\mathcal{F}(x_t^l)$. Under the Lipschitz continuity of the mapping from features to noise predictions, we have:
\begin{equation}
\|\tau_\theta^*(x_t, t) - \tau_\theta(x_t, t)\|_2 \leq K\|\mathcal{F}^*(x_t^l) - \mathcal{F}(x_t^l)\|_2
\end{equation}
where $K$ is the Lipschitz constant of the mapping.

Substituting our definition of relative error $e_t$:
\begin{equation}
\|\mathcal{F}^*(x_t^l) - \mathcal{F}(x_t^l)\|_2 = e_t \cdot (\|\mathcal{F}(x_t^l)\|_2 + \tau)
\end{equation}

Therefore:
\begin{align}
\|x_{t-1} - x_{t-1}^*\|_2 &\leq \frac{1}{\sqrt{\alpha_t}}\frac{1-\alpha_t}{\sqrt{1-\bar{\alpha}_t}}K \cdot e_t \cdot (\|\mathcal{F}(x_t^l)\|_2 + \tau) \\
&= C_t \cdot e_t \cdot (\|\mathcal{F}(x_t^l)\|_2 + \tau)
\end{align}
where $C_t = \frac{1}{\sqrt{\alpha_t}}\frac{1-\alpha_t}{\sqrt{1-\bar{\alpha}_t}}K$.

Our verification condition ensures:
\begin{equation}
e_t \leq \tau_t = \tau_0\cdot\beta^{\frac{T-t}{T}} \leq \sqrt{\frac{\beta_t(1-\bar{\alpha}_t)}{1 + \|\mathcal{F}(x_t^l)\|_2^2}}
\end{equation}

This implies:
\begin{align}
\|x_{t-1} - x_{t-1}^*\|_2 &\leq C_t \cdot \sqrt{\frac{\beta_t(1-\bar{\alpha}_t)}{1 + \|\mathcal{F}(x_t^l)\|_2^2}} \cdot (\|\mathcal{F}(x_t^l)\|_2 + \tau) \\
&\approx C_t \cdot \sqrt{\beta_t(1-\bar{\alpha}_t)} \cdot \frac{\|\mathcal{F}(x_t^l)\|_2}{\sqrt{1 + \|\mathcal{F}(x_t^l)\|_2^2}} \\
&\leq C_t \cdot \sqrt{\beta_t(1-\bar{\alpha}_t)}
\end{align}

This bounds the per-step error in terms of the diffusion noise schedule.

\textbf{2. Martingale Analysis:}

We now analyze the error accumulation across multiple steps. Define the error process:
\begin{equation}
M_t = \|x_t - x_t^*\|_2^2 - \sum_{s=1}^t \mathbb{E}[\|\tau_s\|_2^2]
\end{equation}
where $\tau_s = x_s - \mathbb{E}[x_s|x_{s+1}]$ represents the error introduced at step $s$.

We can show that $\{M_t\}$ forms a martingale with respect to the filtration $\mathcal{F}_t$ (the information available up to time $t$):
\begin{align}
\mathbb{E}[M_{t+1}|\mathcal{F}_t] &= \mathbb{E}[\|x_{t+1} - x_{t+1}^*\|_2^2|\mathcal{F}_t] - \sum_{s=1}^{t+1} \mathbb{E}[\|\tau_s\|_2^2] \\
&= \|x_t - x_t^*\|_2^2 + \mathbb{E}[\|\tau_{t+1}\|_2^2|\mathcal{F}_t]\\ & - \sum_{s=1}^t \mathbb{E}[\|\tau_s\|_2^2] - \mathbb{E}[\|\tau_{t+1}\|_2^2] \\
&= M_t + \mathbb{E}[\|\tau_{t+1}\|_2^2|\mathcal{F}_t] - \mathbb{E}[\|\tau_{t+1}\|_2^2] \\
\end{align}

Under our verification threshold condition, the error $\tau_{t+1}$ is bounded such that:
\begin{equation}
\mathbb{E}[\|\tau_{t+1}\|_2^2|\mathcal{F}_t] = \mathbb{E}[\|\tau_{t+1}\|_2^2]
\end{equation}

Therefore:
\begin{equation}
\mathbb{E}[M_{t+1}|\mathcal{F}_t] = M_t
\end{equation}

This confirms that $\{M_t\}$ is indeed a martingale.

For the martingale to converge, we need to verify the Lindeberg condition:
\begin{equation}
\sum_{t=1}^{T} \frac{\mathbb{E}[\|\tau_t\|_2^4]}{\beta_t^2} < \infty
\end{equation}

Under our threshold schedule, we have:
\begin{align}
\mathbb{E}[\|\tau_t\|_2^4] &\leq C_t^4 \cdot (\beta_t(1-\bar{\alpha}_t))^2 \\
&= \frac{K^4(1-\alpha_t)^4}{\alpha_t^2(1-\bar{\alpha}_t)^2} \cdot (\beta_t(1-\bar{\alpha}_t))^2 \\
&= K^4 \frac{(1-\alpha_t)^4 \beta_t^2}{\alpha_t^2}
\end{align}

Therefore:
\begin{align}
\sum_{t=1}^{T} \frac{\mathbb{E}[\|\tau_t\|_2^4]}{\beta_t^2} &= \sum_{t=1}^{T} K^4 \frac{(1-\alpha_t)^4}{\alpha_t^2} \\
\end{align}

In diffusion models, the schedule typically satisfies $\alpha_t \to 1$ as $t \to T$, ensuring the above sum is finite.

By the martingale convergence theorem, $M_t$ converges almost surely, implying:
\begin{equation}
\lim_{t \to T} \left(\|x_t - x_t^*\|_2^2 - \sum_{s=1}^t \mathbb{E}[\|\tau_s\|_2^2]\right) = M_{\infty} < \infty \text{ a.s.}
\end{equation}

\textbf{3. Convergence in Distribution:}

Finally, we need to translate the almost sure convergence of the error process to convergence in distribution. For diffusion models, the total variation distance between the generated distributions can be bounded by:
\begin{equation}
\|p_{\text{Spec}} - p_{\text{orig}}\|_{\text{TV}} \leq \mathbb{E}[\|x_0 - x_0^*\|_2] \leq \sqrt{\mathbb{E}[\|x_0 - x_0^*\|_2^2]}
\end{equation}

From our martingale analysis:
\begin{equation}
\mathbb{E}[\|x_0 - x_0^*\|_2^2] = \mathbb{E}[M_0] + \sum_{s=1}^{T} \mathbb{E}[\|\tau_s\|_2^2] = \mathbb{E}[M_T] + \sum_{s=1}^{T} \mathbb{E}[\|\tau_s\|_2^2]
\end{equation}

Under our threshold schedule, we have:
\begin{equation}
\sum_{s=1}^{T} \mathbb{E}[\|\tau_s\|_2^2] \leq \sum_{s=1}^{T} C_s^2 \cdot \beta_s(1-\bar{\alpha}_s) \to 0 \text{ as } T \to \infty
\end{equation}

Therefore:
\begin{equation}
\lim_{T\to\infty} \|p_{\text{Spec}} - p_{\text{orig}}\|_{\text{TV}} = 0
\end{equation}

This completes the proof of distributional convergence.
\end{proof}

\subsection{Computational Complexity}
\begin{theorem}[Speedup Lower Bound]
With acceptance rate $\alpha$ and verification cost ratio $\gamma$, the speedup ratio $\mathcal{S}$ satisfies:
\begin{equation}
\mathcal{S} \geq \frac{1}{1 - \alpha(1 - \gamma - \frac{C_{\text{overhead}}}{C})}
\end{equation}
where $C_{\text{overhead}}$ is the constant-time overhead.
\end{theorem}
\begin{proof}
To establish a comprehensive understanding of the computational complexity, we analyze the component-wise breakdown of both the traditional diffusion sampling and our \textit{SpeCa} approach.

\textbf{1. Traditional Diffusion Sampling:}
The standard diffusion model inference requires $T$ sequential denoising steps, each requiring a full forward pass through the model. Let $C$ denote the computational cost of a single forward pass. The total computational complexity is:
\begin{equation}
T_{\text{standard}} = T \cdot C
\end{equation}

\textbf{2. \textit{SpeCa} Computational Components:}

Let us break down the computation in \textit{SpeCa}:

(a) Full computation steps: These are the steps where we perform a complete forward pass through the diffusion model. If we denote the number of such steps as $T_{\text{full}}$, the computational cost is:
\begin{equation}
T_{\text{full}} \cdot C
\end{equation}

(b) Prediction steps: For the speculative steps, we use TaylorSeer to predict features for subsequent timesteps. The computational cost of prediction is negligible compared to a full forward pass, but we include it for completeness:
\begin{equation}
T_{\text{spec}} \cdot C_{\text{pred}}
\end{equation}
where $C_{\text{pred}} \ll C$.

(c) Verification steps: For each speculative prediction, we need to validate its accuracy using our lightweight verification mechanism. This involves computing the relative error between the predicted features and the actual features. The cost per verification is:
\begin{equation}
C_{\text{verify}} = \gamma \cdot C + C_{\text{overhead}}
\end{equation}
where $\gamma$ is the ratio of the verification cost to the full forward pass cost, and $C_{\text{overhead}}$ is a constant overhead independent of the model size.

In modern diffusion architectures (DiT, FLUX, HunyuanVideo), the verification only requires comparing features in the final layer, resulting in $\gamma \ll 1$. Our empirical measurements show $\gamma \approx 0.035$ for DiT, $\gamma \approx 0.0175$ for FLUX, and $\gamma \approx 0.0167$ for HunyuanVideo.

\textbf{3. Total Computation for \textit{SpeCa}:}

The total computational cost for \textit{SpeCa} is:
\begin{align}
T_{\text{total}} &= T_{\text{full}} \cdot C + T_{\text{spec}} \cdot C_{\text{pred}} + T_{\text{spec}} \cdot C_{\text{verify}} \\
&\approx T_{\text{full}} \cdot C + T_{\text{spec}} \cdot (\gamma \cdot C + C_{\text{overhead}})
\end{align}
where we have neglected $C_{\text{pred}}$ as it is significantly smaller than other terms.

\textbf{4. Acceptance Rate Analysis:}

Define the acceptance rate $\alpha$ as the ratio of speculative steps to total steps:
\begin{equation}
\alpha = \frac{T_{\text{spec}}}{T}
\end{equation}

This means $T_{\text{full}} = T - T_{\text{spec}} = T(1-\alpha)$. Substituting into our total computation formula:
\begin{align}
T_{\text{total}} &= T(1-\alpha) \cdot C + T\alpha \cdot (\gamma \cdot C + C_{\text{overhead}}) \\
&= TC(1-\alpha) + T\alpha\gamma C + T\alpha C_{\text{overhead}} \\
&= TC[1-\alpha + \alpha\gamma] + T\alpha C_{\text{overhead}} \\
&= TC[1-\alpha(1-\gamma)] + T\alpha C_{\text{overhead}}
\end{align}

\textbf{5. Speedup Ratio:}

The speedup ratio $\mathcal{S}$ is defined as the ratio of the standard computation time to the \textit{SpeCa} computation time:
\begin{align}
\mathcal{S} &= \frac{T_{\text{standard}}}{T_{\text{total}}} \\
&= \frac{TC}{TC[1-\alpha(1-\gamma)] + T\alpha C_{\text{overhead}}} \\
&= \frac{1}{1-\alpha(1-\gamma) + \frac{T\alpha C_{\text{overhead}}}{TC}} \\
&= \frac{1}{1-\alpha(1-\gamma) + \alpha\frac{C_{\text{overhead}}}{C}} \\
&= \frac{1}{1-\alpha + \alpha\gamma + \alpha\frac{C_{\text{overhead}}}{C}} \\
&= \frac{1}{1-\alpha(1-\gamma-\frac{C_{\text{overhead}}}{C})}
\end{align}

\textbf{6. Practical Approximation:}

For modern architectures where $C \gg C_{\text{overhead}}$ (the model computation dominates over any constant overhead), and with $\gamma \approx 0.05$, we can approximate:
\begin{align}
\mathcal{S} &\approx \frac{1}{1-\alpha(1-\gamma)} \\
&\approx \frac{1}{1-\alpha(0.95)} \\
&\approx \frac{1}{1-0.95\alpha}
\end{align}

When $\alpha$ approaches 1 (high acceptance rate), the speedup approaches $\frac{1}{0.05} = 20\times$. In practice, we observe acceptance rates of approximately $\alpha = 0.85$ across various models and datasets, resulting in speedups around $\frac{1}{1-0.95 \cdot 0.85} \approx \frac{1}{0.1925} \approx 5.2\times$, which aligns with our experimental results.

This analysis demonstrates that \textit{SpeCa}'s efficiency directly correlates with the acceptance rate of speculative predictions, with theoretical maximum speedups approaching $\frac{1}{\gamma}$ as $\alpha \to 1$.
\end{proof}

\subsubsection{Adaptive Threshold Selection}
The adaptive threshold is a critical component in balancing the trade-off between acceleration and quality. We adopt a timestep-dependent threshold schedule:
\begin{equation}
\tau_t = \tau_0 \cdot \beta^{\frac{T-t}{T}}
\end{equation}
where $\tau_0$ is the initial threshold, $\beta \in (0, 1)$ is the decay factor, and $T$ is the total number of timesteps.

This schedule is motivated by the following observations:

1. \textbf{Noise Level Correlation}: In early timesteps (large $t$), the diffusion process is dominated by random noise, making feature predictions more reliable and allowing for higher thresholds.

2. \textbf{Detail Sensitivity}: As $t$ decreases, the diffusion process focuses on generating fine details, requiring stricter thresholds to maintain visual quality.

3. \textbf{Error Propagation}: Errors in early timesteps have greater impact on final quality due to error accumulation, necessitating an exponential decay in threshold value.

\end{document}